\documentclass{article}
\usepackage{arxiv}

\usepackage{algorithmic}
\usepackage{rotating}
\usepackage[caption=false,font=normalsize,labelfont=sf,textfont=sf]{subfig}
\usepackage{stfloats}
\usepackage{verbatim}
\usepackage{graphicx}%
\usepackage{multirow}%
\usepackage{amsmath,amssymb,amsfonts}%
\usepackage{amsthm}%
\usepackage{mathrsfs}%
\usepackage{textcomp, cancel}%
\usepackage{manyfoot}%
\usepackage{algorithm}%
\usepackage{listings}%
\usepackage{hyperref}       
\usepackage{url}            
\usepackage{booktabs}       
\usepackage{arydshln}
\usepackage{microtype}      
\usepackage{placeins}
\usepackage{longtable}
\usepackage{subcaption}
\usepackage{cite}
\usepackage{array}
\usepackage{multirow}
\usepackage{bm}

\def\BibTeX{{\rm B\kern-.05em{\sc i\kern-.025em b}\kern-.08em
    T\kern-.1667em\lower.7ex\hbox{E}\kern-.125emX}}
\usepackage{balance}

\newtheorem{theorem}{Theorem}
\newtheorem{proposition}[theorem]{Proposition}

\newcommand{\x}{\boldsymbol{x}}
\newcommand{\bx}{\boldsymbol{x}}
\newcommand{\y}{\boldsymbol{y}}
\newcommand{\by}{\boldsymbol{y}}
\newcommand{\z}{\boldsymbol{z}}
\newcommand{\e}{\boldsymbol{e}}
\newcommand{\bd}{\boldsymbol{d}}
\newcommand{\m}{\boldsymbol{m}}
\newcommand{\xot}{\hat{\x}_{0\vert t}}
\newcommand{\xt}{\x_{t}}
\newcommand{\K}{\boldsymbol{K}}
\newcommand{\D}{\boldsymbol{D}}
\newcommand{\I}{\boldsymbol{I}}
\newcommand{\eps}{\bm{\epsilon}}

\usepackage[dvipsnames]{xcolor}

\begin{document}

\title{Consecutive Posterior Fusion for Diffusive Recovery of Unobservable Image Structures}

\author{Elena Morotti\\
Department of Computer Science and Engineerin \\
University of Bologna \\
Bologna, 40126, Italy \\
\texttt{elena.morotti4@unibo.it} \and
Davide Evangelista \\
Department of Computer Science and Engineering\\
University of Bologna \\
Bologna, 40126, Italy \\
\texttt{davide.evangelista5@unibo.it} \And
Elena Loli Piccolomini \\
Department of Computer Science and Engineering\\
University of Bologna \\
Bologna, 40126, Italy \\
\texttt{elena.loli@unibo.it}}

\maketitle

\begin{abstract}
Solving severely ill-posed imaging inverse problems requires recovering image structures that are unobservable or weakly constrained by the measurements. Diffusion models provide expressive learned priors for inferring such missing information, while posterior sampling incorporates measurement consistency along the reverse process. 
Standard diffusion posterior samplers, however, rely on instantaneous measurement-aware estimates, without explicitly exploiting information carried by previous posterior corrections.

We introduce Consecutive Posterior Fusion Denoising Diffusion Null-Space Models (CPF-DDNM), an inference-time strategy that fuses consecutive measurement-aware estimates to improve the diffusive recovery of unobservable image structures, without requiring retraining or additional denoiser evaluations.
We instantiate this principle within DDNM, whose range/null-space decomposition reveals that consecutive fusion preserves the measurement-determined component while acting exclusively on the prior-driven null-space estimate.
We thus provide a geometric interpretation of CPF-DDNM and a local error analysis that characterizes the optimal time-dependent fusion coefficient, including the extrapolative regime.

Experiments on sparse-view and simulated low-dose computed tomography, as well as medical image super-resolution, show consistent improvements over DDNM and competitive performance against diffusion-based inverse solvers.
\end{abstract}


\section{Introduction}

Imaging inverse problems aim to recover an unknown image from indirect
measurements related to it through a physical acquisition process. In many
practical imaging settings, these measurements may be incomplete, undersampled,
or affected by noise, making the reconstruction problem severely ill-posed \cite{bertero2021introduction}.
Such situations arise, for example, in computed tomography (CT), magnetic
resonance imaging (MRI), inpainting, and super-resolution.
In these settings, the acquisition process may not uniquely determine the image:
undersampling leaves components that are unobservable from the data, while
measurement noise further increases the uncertainty along weakly observable
directions of the forward operator. The quality of the reconstruction therefore
depends critically on the prior used to infer the missing or poorly constrained
information.

Consider the noiseless linear imaging  model $\y=\K\x$, where $\K$ is the forward operator modelling the imaging system, $\y$ contains the acquired data and $\x$ is the image to be reconstructed. Whenever $\K$ has a non-trivial null-space $\operatorname{Ker}(\K)$ and $\y\in\operatorname{Range}(\K)$, the set of data-consistent solutions is the affine space 
\begin{equation}\label{eq:model}
    \x = \K^\dagger\y+\z, \qquad \z\in\operatorname{Ker}(\K),
\end{equation}
where $\K^\dagger$ denotes the pseudo-inverse of the forward operator, $\K^\dagger\y$ is the minimum-norm solution determined by the measurements, and $\z$ is invisible to the acquisition operator. Infinitely many data-consistent reconstructions can thus be obtained by varying $\z$. In the presence of noise, this ambiguity remains, together with the additional uncertainty associated with small singular values of $\K$.

Classical variational methods resolve this ambiguity through handcrafted regularizers that promote properties such as sparsity, smoothness, or structural consistency \cite{engl1996regularization,hansen1998rank,hansen2006deblurring}. Deep learning provides an alternative route by learning image statistics directly from data.
In medical reconstruction, learned components have been employed as post-processing operators \cite{jin2017deep,morotti2021green}, embedded in unrolled optimization schemes \cite{gupta2018cnn,loli2023ctprintnet}, or combined with explicit model-based procedures \cite{morotti2026adaptive}.  More recently, diffusion generative models have emerged as expressive learned priors for inverse problems. Among the broad family of approaches employing diffusion models to drive the image reconstruction problem, in this work we focus on diffusion posterior sampling methods. These methods typically combine a pretrained generative prior with measurement-aware corrections throughout the reverse trajectory, thereby selecting plausible solutions that remain compatible with the observed data. This is done by adjusting the clean-image prediction $\xot$ obtained by the diffusion model at timestep $t$ using a measurement-aware term, producing a correction $\D_t$, which is then inserted into the reverse diffusion process in place of the original $\xot$. An important representative of the class of diffusion posterior samplers is the Denoising Diffusion Null-Space Model (DDNM) \cite{wangzero}, which is particularly relevant because it makes the separation between measured and prior-inferred information explicit by defining its correction as
\begin{equation}
    \D_t = \K^\dagger \y + (\I-\K^\dagger \K)\xot,
    \label{eq:Dt_intro}
\end{equation}
where $\K^\dagger\y$ is the component determined by the measurements and $(\I-\K^\dagger\K)\xot$ is the null-space completion inferred by the learned prior. This structure is especially relevant for severely undersampled medical imaging problems, such as sparse-view CT and high-factor super-resolution, whose solutions depend strongly on the recovery of unobservable image components.
However, a limitation of classical diffusion posterior sampling methods is that, particularly during the early timesteps, when noise dominates the reverse diffusion process, the correction terms $\D_t$ are volatile, sometimes causing the process to diverge towards an out-of-distribution solution. In our recent Lagged Multistep Posterior (LAMP) method \cite{evangelista2026improving}, we showed that consecutive measurement-aware estimates are highly correlated and may therefore be combined to reduce their variance and stabilize the reconstruction trajectory. While LAMP provides a general mechanism, based on a multistep approximation of an underlying differential equation, for exploiting this interaction, it does not investigate how consecutive posterior corrections interact with the specific structure of the underlying inverse problem.

In this work, we address this question within DDNM, whose explicit range/null-space decomposition makes it possible to identify where the use of consecutive posteriors acts.
To this aim, we define
\begin{align}\label{eq:fusion_intro}
    \widetilde{\D}_t
    := \D_t + \beta_t
    \left(\D_{t+\Delta t}-\D_t\right),
\end{align}
where $\D_{t+\Delta t}$ is the correction produced at the preceding reverse iteration and $\beta_t$ is the time-dependent fusion coefficient.
We refer to this general corrective strategy as Consecutive Posterior Fusion (CPF) and we denote its application to DDNM as CPF-DDNM. In the following, we show that consecutive DDNM corrections $\D_{t+\Delta t}$ and $\D_{t}$ share the same measurement-determined component $\K^\dagger\y$. Therefore, the CPF correction $\widetilde{\D}_t$ in Eq.~\eqref{eq:fusion_intro} preserves the component determined by the measurements and modifies only the prior-driven recovery of the unobservable image component.

The range/null-space structure of DDNM, in addition, provides a direct geometric interpretation of the fusion mechanism and allows the role of the fusion coefficient to be analyzed specifically in terms of null-space recovery, including both interpolative ($0 < \beta_t < 1$) and extrapolative ($\beta_t > 1$) choices.
In particular, we investigate time-dependent coefficients that can move beyond the segment connecting consecutive null-space estimates while leaving the measurement-determined component unaffected.


The main contributions of this work are summarized as follows:
\begin{itemize}
    \item We specialize consecutive posterior fusion to the range/null-space decomposition of DDNM. We show that the fusion preserves the measurement-determined component and modifies exclusively the completion inferred by the diffusion prior.

    \item We study the extrapolative regime, interpreting this as an over-relaxation within the null-space, and derive a local optimality condition showing when an extrapolative coefficient can minimize the expected completion error.

    \item We introduce time-dependent schedules for $\{ \beta_t \}_{t \in [0, T]}$, and evaluate CPF-DDNM on simulated sparse-view CT, low-dose CT, and medical image super-resolution.
    CPF-DDNM improves the original DDNM baseline across all experiments, without retraining the diffusion prior or adding denoiser evaluations.
    The results also demonstrate competitive or superior performance with respect to other state-of-the-art diffusion-based inverse solvers.
\end{itemize}

The remainder of the paper is organized as follows.
Section~\ref{sec:related_works} reviews diffusion-based methods for imaging inverse problems. Section~\ref{sec:method} introduces CPF-DDNM and presents its geometric interpretation and theoretical analysis. Section~\ref{sec:expsetup} describes the experimental protocol, and Section~\ref{sec:results} presents and discusses the results. Finally, Section~\ref{sec:conclusions} summarizes the main findings and outlines directions for future work.

\section{Related Works}\label{sec:related_works}
Diffusion models have been widely adopted as learned priors for imaging inverse problems. Existing approaches differ primarily in how measurement information is incorporated into the reverse process. DDRM \cite{kawar2022denoising} derives posterior updates for linear degradation operators through their spectral decomposition. DPS \cite{chung2023diffusion} introduces approximate likelihood gradients and can address noisy and nonlinear inverse problems. DDNM \cite{wangzero} enforces measurement consistency through an explicit range/null-space decomposition, whereas DiffPIR \cite{zhu2023denoising} embeds a pretrained diffusion denoiser within a plug-and-play half-quadratic splitting scheme. Although these approaches implement data consistency differently, they all generate a sequence of measurement-aware estimates along the reverse
trajectory.

To speed up the generation process or improve the stability of the diffusion trajectory, several works have recently investigated multistep discretizations of the reverse diffusion dynamics
\cite{liu2022pndm,lu2022dpmsolver++,zhang2023fast}. These approaches exploit information from multiple diffusion steps to construct higher-order approximations of the generative trajectory. By comparison, the temporal relation between consecutive \emph{measurement-aware} estimates in diffusion posterior samplers remains comparatively underexplored.

A recent step in this direction is LAMP \cite{evangelista2026improving}, which introduces a lagged temporal correction derived from a second-order discretization of the posterior dynamics. LAMP provides initial evidence that consecutive measurement-aware estimates contain useful information that is discarded by standard posterior-sampling updates. Its formulation is backbone-independent and was validated on DiffPIR and DDRM using a linear combination of consecutive estimates. However, due to the second-order discretization formulation, it always resulted in a convex combination of successive measurement-aware terms with coefficients $\beta_t \in [0,1]$. 
On the other side, CPF-DDNM  focuses on the extrapolative regime, where consecutive estimates define an over-relaxation direction within the null-space. Thus, the two methods exploit the same source of temporal information in different regimes and for different purposes.

Diffusion models have also attracted growing interest in medical image reconstruction, where they provide flexible learned priors for highly ill-posed acquisition problems
\cite{kazerouni2023diffusion,medical_diffusion_reconstruction_review}. Applications include sparse-view and low-dose CT, accelerated MRI, PET reconstruction, and multimodal imaging. In CT, diffusion-based methods have been proposed for sparse-view reconstruction and artifact reduction, including approaches that incorporate the acquisition model directly into the generative process \cite{patch_diffusion_ct,ct_sdm}. Related posterior-sampling strategies have been developed for accelerated MRI and joint PET--MRI reconstruction \cite{inr_diffusion_mri,joint_diffusion_pet_mri}. Most of these works focus on modality-specific conditioning mechanisms, data-consistency operators, or posterior-sampling rules. CPF-DDNM is complementary to these developments: it modifies how an existing DDNM solver uses consecutive measurement-aware estimates, without retraining or changing the underlying diffusion prior.

\section{Method}\label{sec:method}
In this section, we present the proposed posterior fusion strategy for diffusion-based medical image reconstruction. The proposed CPF-DDNM method builds on the null-space formulation of DDNM and replaces the posterior estimate computed at the current reverse iteration with a fusion of two consecutive posterior estimates. 

\subsection{Diffusion models for image generation}\label{ssec:ddim}

Diffusion models constitute one of the most effective paradigms for deep learning-based image generation. In the context of diffusion posterior sampling, we focus on Denoising Diffusion Implicit Models (DDIMs) \cite{DDIM}, whose sampling formulation provides the basis for the reconstruction methods considered in this work. We briefly review their generative mechanism to establish the notation used throughout the paper.

Let $\x_0$ be a sample from the distribution $p_0(\x)$ of clean images. Given two sequences $\{ \alpha_t \}_{t=1}^T$ and $\{ \sigma_t \}_{t=1}^T$ satisfying $\alpha_t^2 + \sigma_t^2 = 1$, respectively referred to as the signal-rate and noise-rate schedules, define the stochastic process $\{ \x_t \}_{t=1}^T$ as
\begin{align}\label{eq:ddim_forward}
    \x_t = \alpha_t \x_0 + \sigma_t \eps_t, \qquad \eps_t \sim \mathcal{N}(\boldsymbol{0}, \I).
\end{align}
When $\alpha_T \approx 0$, the terminal variable approximately follows a standard Gaussian distribution, i.e. $\x_T \sim \mathcal{N}(\boldsymbol{0}, \I)$. A neural network-based denoiser $\eps_\Theta(\x_t, t)$ is trained to predict the noise component $\eps_t$ from the noisy sample $\x_t$ by minimizing
\begin{align}\label{eq:loss_function}
    \mathcal{L}(\Theta) = \mathbb{E}_{\x_0 \sim p_0, t \sim \mathcal{U}[0, T], \eps_t \sim \mathcal{N}(\boldsymbol{0}, \I)} \left[ || \eps_\Theta(\x_t, t) - \eps_t ||_2^2 \right]
\end{align}
Once trained, the denoiser can be used to progressively generate samples from $p_0(\x)$ by reversing the process in Eq. \eqref{eq:ddim_forward}.

In particular, the clean sample can be approximated through the Tweedie estimate
\begin{align}\label{eq:ddim_xot}
    \xot := \frac{\x_t - \sigma_t \eps_\Theta(\x_t, t)}{\alpha_t},
\end{align}
which follows by isolating $\x_0$ in Eq. \eqref{eq:ddim_forward} and replacing the unknown noise $\eps_t$ with its neural-network estimate. The quantity $\xot$ therefore represents an estimate of the sample $\x_0$ based on the information available at timestep $t$. As the reverse process progresses from $T$ towards $0$, this estimate typically becomes increasingly detailed and coherent.

Starting from Eq. \eqref{eq:ddim_xot}, DDIM constructs the sample at the next reverse timestep by reintroducing noise according to the lower noise level $\sigma_{t-\Delta t}$:
\begin{align}\label{eq:ddim_backward}
    \x_{t-\Delta t} = \alpha_{t-\Delta t} \xot + \sigma_{t - \Delta t} \eps_\Theta(\x_t, t),
\end{align}
where $\Delta t$ denotes the temporal step size of the reverse process. By repeatedly applying Eq. \eqref{eq:ddim_xot} and Eq. \eqref{eq:ddim_backward}, starting from $\x_T \sim \mathcal{N}(\boldsymbol{0}, \I)$ and proceeding backwards in time until $\x_0$, DDIM generates a new sample from $p_0(\x)$. The parameter $\Delta t > 0$ determines the number of denoising iterations, since generating $\x_0$ from $\x_T$ requires $N = \frac{T}{\Delta t}$ reverse steps. A larger number of steps generally increases the computational cost while providing a finer discretization of the reverse trajectory and, potentially, a more accurate generated sample.

\subsection{Diffusion posterior samplers}\label{ssec:diffusion-posterior-sampler}
The generative process described in Section \ref{ssec:ddim} can produce high-quality samples, although it does not directly control which image is generated. Solving an inverse problem instead requires the generated image to remain compatible with the observed measurements. 
Diffusion posterior samplers address this requirement by introducing a measurement-aware correction at each reverse timestep, immediately after computing the clean estimate $\xot$.
This correction progressively guides the reverse trajectory towards consistency with the measurements, while the diffusion prior complements the available information by estimating the unobservable component $\z \in \operatorname{Ker}(\K)$ introduced in Eq.~\eqref{eq:model}.

Consider a pre-trained diffusion model $\eps_\Theta(\x_t, t)$, trained in an unsupervised manner on a dataset representative of the images to be processed. Given an initial iterate $\x_T \sim \mathcal{N}(\boldsymbol{0}, \I)$, the clean estimate $\xot$ is computed at each timestep according to Eq. \eqref{eq:ddim_xot}. Then, a measurement-aware correction operator $\mathcal{C}(\xot, t, \K, \y)$ is introduced, whose formulation depends on the selected posterior-sampling method. We call the corrected image $\D_t$, i.e. $\D_t := \mathcal{C}(\xot, t, \K, \y)$.
Replacing $\xot$ with $\D_t$ in Eq. \eqref{eq:ddim_backward} yields the diffusion posterior-sampling update
\begin{align}\label{eq:diffusion-posterior-step}
    \x_{t-\Delta t} = \alpha_{t-\Delta t} \D_t + \sigma_{t-\Delta t} \eps_\Theta(\x_t, t),
\end{align}
through which the diffusion trajectory is progressively conditioned on the acquired measurements. 
Iterating Eq.~\eqref{eq:diffusion-posterior-step} until $t=\Delta t$ yields the final reconstruction $\x_0$. 

Several posterior samplers can be interpreted within this measurement-correction framework, and the specific choice of $\mathcal{C}(\xot, t, \K, \y)$ determines the diffusion posterior-sampling method. In addition, the precise point at which the correction is introduced further differs across methods. 
Focusing on the DDNM method proposed in \cite{wangzero}, it adopts the correction described in Eq. \eqref{eq:Dt_intro}. Denoting the projection operator onto the kernel of $\K$ as $\mathcal{P}_{\operatorname{Ker}} := \I-\K^\dagger \K$ \cite{engl1996regularization}, the DDNM correction reads:
\begin{equation}
    \D_t = \K^\dagger \y + \mathcal{P}_{\operatorname{Ker}} \xot. 
    \label{eq:DDNM}
\end{equation}
For the noiseless setting, this expression provides a clear geometric interpretation of the correction of $\xot$:  
the DDNM preserves only the estimated null-space component $\mathcal{P}_{\operatorname{Ker}} \xot$, while it replaces the range-space component of $\xot$ with the measurement-determined term $\K^\dagger \y$. 
When $\K^\dagger$ is the exact Moore-Penrose pseudo-inverse and the measurements belong to the range of $\K$, this construction enforces exact data consistency at every corrected estimate. \\
Additionally, this explicit separation also makes the role of the diffusion model particularly transparent, since the learned prior affects only the image components that cannot be recovered directly from the measurements. 

Applying the same correction directly to noisy measurements may, however, amplify noise in the reconstruction. In an ill-conditioned problem, the pseudo-inverse strongly amplifies measurement components associated with small singular values of $\K$, and the term $\K^\dagger \y$ may consequently become unstable. 

\subsection{The CPF-DDNM method}
\label{ssec:cpf_ddnm}

Although the reverse diffusion process in Eq.~\eqref{eq:diffusion-posterior-step} generates a sequence of strongly correlated estimates of the same unknown image, standard DDNM constructs its correction using exclusively the denoised estimate produced at the current timestep. The estimate obtained at the preceding reverse iteration is therefore discarded, even though it provides a closely related prediction of the same unobservable image component. Motivated by this observation, we propose to combine the DDNM corrections obtained at two consecutive reverse iterations. As anticipated in Eq.~\eqref{eq:fusion_intro}, the resulting fused correction is defined as $\widetilde{\D}_t=\D_t+\beta_t(\D_{t+\Delta t}-\D_t)$, where $\beta_t\in\mathbb{R}$ controls the contribution of the preceding correction. By expanding the two DDNM terms, this expression becomes
\begin{equation}
\begin{split}
\widetilde{\D}_t
&= \K^\dagger\y+\left(\I-\K^\dagger\K\right)\Bigl(\xot+\beta_t\left(\hat{\x}_{0\mid t+\Delta t}-\xot\right)\Bigr) \\
&= \K^\dagger\y+\mathcal{P}_{\operatorname{Ker}}\Bigl(\xot+\beta_t\left(\hat{\x}_{0\mid t+\Delta t}-\xot\right)\Bigr),
\end{split}
\label{eq:fusion_ddnm}
\end{equation}
where $\mathcal{P}_{\operatorname{Ker}}:=\I-\K^\dagger\K$. The proposed reverse step is then obtained by replacing the standard DDNM correction $\D_t$ with its fused counterpart $\widetilde{\D}_t$ in Eq.~\eqref{eq:diffusion-posterior-step}, yielding
\begin{equation}
\x_{t-\Delta t}=\alpha_{t-\Delta t}\widetilde{\D}_t+\sigma_{t-\Delta t}\epsilon_\theta(\xt,t).
\label{eq:reverse_fusion}
\end{equation}
This extension of DDNM, which we refer to as \emph{Consecutive Posterior Fusion DDNM} (CPF-DDNM), requires only the storage of the correction obtained at the preceding reverse iteration. It introduces no additional denoiser evaluations and requires no modification or retraining of the pretrained diffusion model.

Equation~\eqref{eq:fusion_ddnm} also shows that CPF-DDNM retains the range/null-space structure of the original DDNM correction. To make this property explicit, let
\begin{equation}
\z_t:=\mathcal{P}_{\operatorname{Ker}}\xot
\label{eq:null_component}
\end{equation}
denote the null-space completion predicted at reverse timestep $t$. The standard DDNM correction can then be written as $\D_t=\K^\dagger\y+\z_t$. Similarly, the correction obtained at the preceding reverse iteration is $\D_{t+\Delta t}=\K^\dagger\y+\z_{t+\Delta t}$, and the CPF-DDNM correction admits the decomposition
\begin{equation}
\widetilde{\D}_t=\K^\dagger\y+\widetilde{\z}_t,
\qquad
\widetilde{\z}_t:=\z_t+\beta_t\left(\z_{t+\Delta t}-\z_t\right).
\label{eq:fused_null_component}
\end{equation}
Hence, the measurement-determined component $\K^\dagger\y$ remains unchanged, while the fusion acts exclusively on the prior-driven null-space completion. The following proposition formalizes this property and establishes the corresponding data-consistency result.

\begin{proposition}[Null-space preservation and data consistency]
\label{prop:null_preservation}
Assume that $\K^\dagger$ is the Moore-Penrose pseudo-inverse of $\K$, so that $\mathcal{P}_{\operatorname{Ker}}=\I-\K^\dagger\K$ is the orthogonal projector onto $\operatorname{Ker}(\K)$. Then, for any $\beta_t\in\mathbb{R}$, the CPF-DDNM correction satisfies
\begin{equation}
\widetilde{\D}_t=\K^\dagger\y+\widetilde{\z}_t,
\qquad
\widetilde{\z}_t\in\operatorname{Ker}(\K),
\end{equation}
and consequently
\begin{equation}
\K\widetilde{\D}_t=\K\K^\dagger\y.
\end{equation}
If $\y\in\operatorname{Range}(\K)$, then $\K\widetilde{\D}_t=\y$, and the fused correction is exactly data-consistent.
\end{proposition}

\begin{proof}
Substituting $\D_t=\K^\dagger\y+\z_t$ and $\D_{t+\Delta t}=\K^\dagger\y+\z_{t+\Delta t}$ into Eq.~\eqref{eq:fusion_intro} gives
\begin{equation}
\widetilde{\D}_t
=\K^\dagger\y+\z_t+\beta_t\left(\z_{t+\Delta t}-\z_t\right)
=\K^\dagger\y+\widetilde{\z}_t.
\end{equation}
Since $\z_t,\z_{t+\Delta t}\in\operatorname{Ker}(\K)$ and $\operatorname{Ker}(\K)$ is a linear subspace, their affine combination $\widetilde{\z}_t=\z_t+\beta_t(\z_{t+\Delta t}-\z_t)$ also belongs to $\operatorname{Ker}(\K)$. It follows that $\K\widetilde{\z}_t=\boldsymbol{0}$ and therefore $\K\widetilde{\D}_t=\K\K^\dagger\y$. When $\y\in\operatorname{Range}(\K)$, the Moore-Penrose identity $\K\K^\dagger\y=\y$ establishes exact data consistency.
\end{proof}

Proposition~\ref{prop:null_preservation} assumes access to the exact Moore-Penrose pseudo-inverse. In the CT experiments considered in this work, the required pseudo-inverse operations are approximated through a finite number of CGLS iterations. Consequently, the corresponding null-space projection and data-consistency properties hold up to the numerical error introduced by this approximation.

\subsection{Insights on the fusion parameter}

The role of $\beta_t$ can now be interpreted geometrically within the null-space. We define the displacement between two consecutive null-space estimates as
\begin{equation}
\bd_t:=\z_{t+\Delta t}-\z_t.
\label{eq:null_fusion_direction}
\end{equation}
Using this notation, Eq.~\eqref{eq:fused_null_component} becomes
\begin{equation}
\widetilde{\z}_t=\z_t+\beta_t\bd_t.
\label{eq:null_fusion_line}
\end{equation}
Hence, the pair $(\z_t,\z_{t+\Delta t})$ determines a local direction $\bd_t$ within the null-space, while $\beta_t$ determines the position of the fused estimate along the affine line passing through these two estimates. Since the measurement-dependent component $\K^\dagger\y$ is unchanged, this displacement affects exclusively the component that is not determined by the measurements.

For $0\leq\beta_t\leq1$, the fused estimate lies on the segment connecting $\z_t$ and $\z_{t+\Delta t}$, and the fusion is therefore interpolative. In particular, $\beta_t=0$ recovers the current estimate $\z_t$, whereas $\beta_t=1$ selects the estimate $\z_{t+\Delta t}$ obtained at the preceding reverse iteration. When $\beta_t>1$, the fused estimate lies beyond $\z_{t+\Delta t}$ along the direction $\bd_t$, and the update becomes extrapolative. In this regime, Eq.~\eqref{eq:null_fusion_line} can equivalently be written as
\begin{equation}
\widetilde{\z}_t=\z_{t+\Delta t}+\m_t,
\label{eq:overrelaxed_fusion}
\end{equation}
where
\begin{equation}
\m_t:=(\beta_t-1)\bd_t=(\beta_t-1)(\z_{t+\Delta t}-\z_t).
\label{eq:momentum_term}
\end{equation}
The quantity $\m_t$ represents the additional displacement introduced beyond $\z_{t+\Delta t}$ and can therefore be interpreted as an over-relaxation term acting exclusively within $\operatorname{Ker}(\K)$. It also has a momentum-like algebraic form, since it extends the direction identified by the difference between two consecutive posterior estimates. The geometric interpretation of this extrapolative regime is illustrated in Figure~\ref{fig:graphical_abstract}.

\begin{figure}[t]
    \centering
    \includegraphics[width=0.8\linewidth]{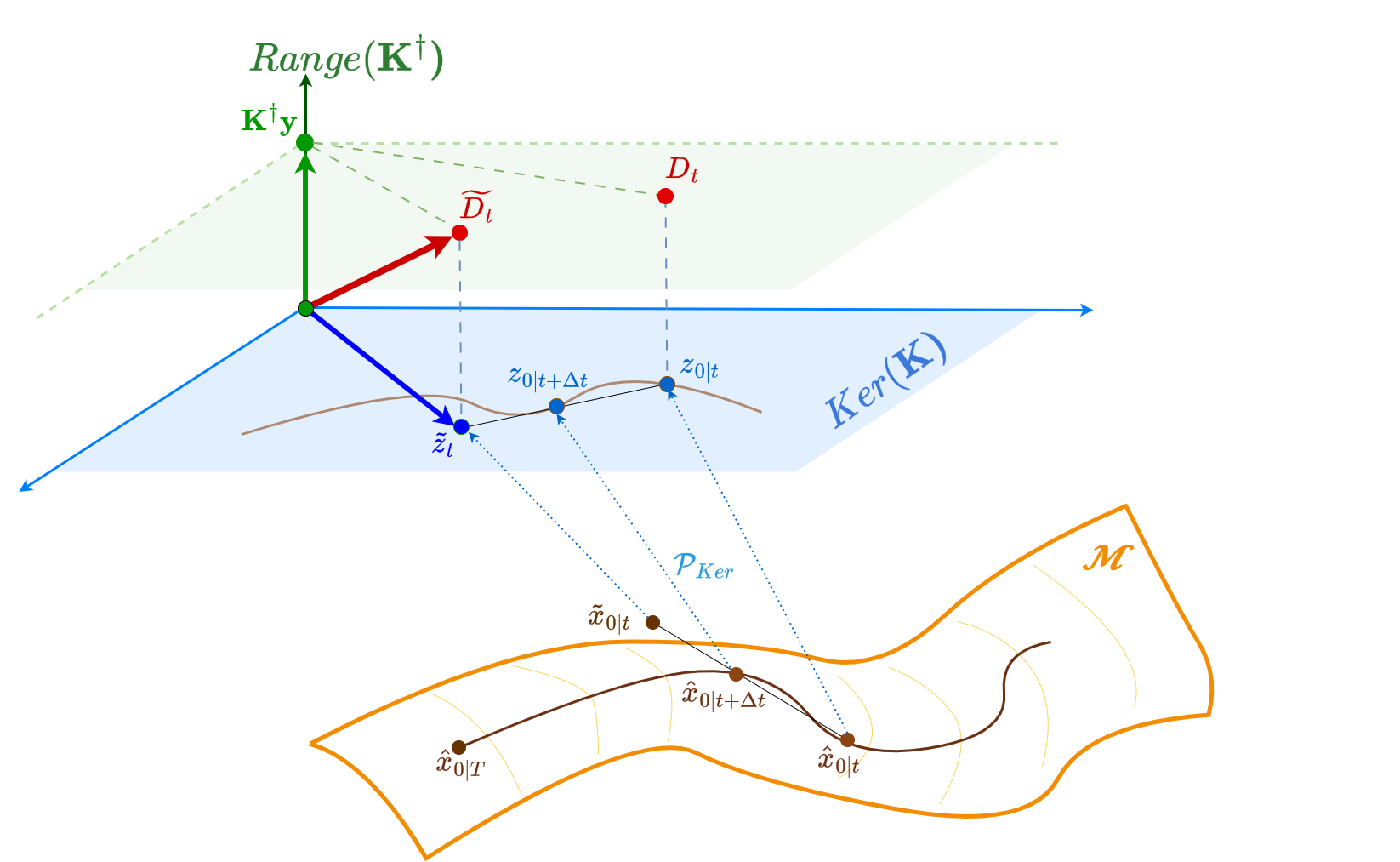}
    \caption{Geometric interpretation of the CPF-DDNM fusion in the extrapolative regime $\beta_t>1$.  The current and preceding null-space estimates $\mathcal{P}_{\operatorname{Ker}}\xot$ and $\mathcal{P}_{\operatorname{Ker}}\hat{\x}_{0 | t + \Delta t}$ determine the direction $\bd_t:=\mathcal{P}_{\operatorname{Ker}}\hat{\x}_{0 | t + \Delta t} - \mathcal{P}_{\operatorname{Ker}}\xot$. CPF-DDNM extends the estimate beyond $\mathcal{P}_{\operatorname{Ker}}\hat{\x}_{0 | t + \Delta t}$ through the additional contribution $\m_t:=(\beta_t-1)\bd_t$, while the measurement-determined component $\K^\dagger\y$ remains unchanged.}
    \label{fig:graphical_abstract}
\end{figure}

The geometric representation above describes how the fused estimate moves as $\beta_t$ varies, while the potential advantage of this operation can be analyzed through a local uncertainty model. Let $\z^*$ denote the ideal null-space component associated with the unknown image, and represent two consecutive estimates as
\begin{equation}
\z_t=\z^*+\boldsymbol{\eta}_t, \qquad \z_{t+\Delta t}=\z^*+\boldsymbol{\eta}_{t+\Delta t},
\label{eq:null_uncertainty_model}
\end{equation}
where $\boldsymbol{\eta}_t$ and $\boldsymbol{\eta}_{t+\Delta t}$ denote the errors associated with the corresponding prior-driven estimates. In the CT experiments, these terms may additionally include perturbations introduced by the numerical approximation of the pseudo-inverse. Substituting Eq.~\eqref{eq:null_uncertainty_model} into Eq.~\eqref{eq:fused_null_component} gives
\begin{equation}
\widetilde{\z}_t-\z^*=(1-\beta_t)\boldsymbol{\eta}_t+\beta_t\boldsymbol{\eta}_{t+\Delta t}.
\label{eq:fused_null_error}
\end{equation}
CPF-DDNM can therefore be interpreted as combining two correlated estimates of the same unknown null-space component. In particular, the error of the fused estimate is itself a linear combination of the errors associated with the two consecutive posterior estimates, suggesting that an appropriate choice of $\beta_t$ can reduce the expected reconstruction error.

For a fixed reverse timestep $t$, we formalize this observation by defining
\begin{equation}
J_t(\beta):=\mathbb{E}\left[\left\|(1-\beta)\boldsymbol{\eta}_t+\beta\boldsymbol{\eta}_{t+\Delta t}\right\|_2^2\right],
\label{eq:fusion_expected_error}
\end{equation}
where the expectation is taken over the randomness inducing the reconstruction errors at the considered timestep. Let
\begin{equation}
\begin{gathered}
A_t := \mathbb{E}\left[\|\boldsymbol{\eta}_t\|_2^2\right], \qquad
B_t := \mathbb{E}\left[\|\boldsymbol{\eta}_{t+\Delta t}\|_2^2\right],\\
C_t := \mathbb{E}\left[\left\langle \boldsymbol{\eta}_t, \boldsymbol{\eta}_{t+\Delta t} \right\rangle \right].
\end{gathered}
\label{eq:error_statistics}
\end{equation}
The following result characterizes the fusion coefficient minimizing the expected error under this local model.

\begin{proposition}[Optimal fusion coefficient]
\label{prop:optimal_beta}
Assume that the consecutive null-space estimates satisfy Eq.~\eqref{eq:null_uncertainty_model}. If $A_t+B_t-2C_t>0$, then $J_t(\beta)$ is strictly convex and admits the unique minimizer
\begin{equation}
\beta_t^*=\frac{A_t-C_t}{A_t+B_t-2C_t}.
\label{eq:optimal_beta_statistical}
\end{equation}
\end{proposition}

\begin{proof}
Expanding Eq.~\eqref{eq:fusion_expected_error} gives
\begin{equation}
J_t(\beta)=(1-\beta)^2A_t+\beta^2B_t+2\beta(1-\beta)C_t.
\label{eq:fusion_error_quadratic}
\end{equation}
Differentiating with respect to $\beta$ yields
\begin{equation}
\frac{\partial J_t}{\partial\beta}=2(C_t-A_t)+2\beta(A_t+B_t-2C_t).
\end{equation}
Setting the derivative equal to zero gives Eq.~\eqref{eq:optimal_beta_statistical}. Moreover,
\begin{equation}
\frac{\partial^2J_t}{\partial\beta^2}=2(A_t+B_t-2C_t)>0,
\end{equation}
which proves that $\beta_t^*$ is the unique minimizer.
\end{proof}

Proposition~\ref{prop:optimal_beta} shows that the optimal fusion coefficient depends on both the error magnitude of the consecutive estimates and their correlation. Its value is therefore not restricted to the interpolation interval $[0,1]$. In particular, under the condition of Proposition~\ref{prop:optimal_beta},
\begin{equation}
\beta_t^*>1 \quad\Longleftrightarrow\quad C_t>B_t,
\label{eq:optimal_extrapolation_condition}
\end{equation}
showing that an extrapolative update may minimize the expected error when the cross-correlation between the current and preceding errors is sufficiently large relative to the error energy of the preceding estimate. This result provides a statistical counterpart to the geometric interpretation in Figure~\ref{fig:graphical_abstract}, while avoiding the need to introduce a separate realization-wise optimization criterion.

The statistics $A_t$, $B_t$, and $C_t$ depend on the unknown reconstruction errors and are therefore unavailable during inference. Rather than attempting to evaluate Eq.~\eqref{eq:optimal_beta_statistical} explicitly, we parametrize $\beta_t$ through a task-dependent schedule of the form
\begin{equation}
\beta_t=c_1+c_2\sqrt{1-\frac{t}{T}},
\label{eq:beta_schedule}
\end{equation}
where $c_1$ and $c_2$ control the evolution of the fusion strength along the reverse trajectory. In particular,
\begin{equation}
\beta_T=c_1, \qquad \beta_0=c_1+c_2.
\end{equation}
Different choices of $c_1$ and $c_2$ therefore allow the fusion strength to increase or decrease during sampling, to vanish as the reverse process approaches the clean-image regime, or to remain active throughout the trajectory, including in the extrapolative regime. The schedules considered in this work and their effect on reconstruction performance are investigated experimentally in Section~\ref{ssec:ablation_beta}.

\section{Experimental setup \label{sec:expsetup}}

In this section, we describe the implementation details of the proposed method and the medical imaging applications whose results are presented and discussed in the following section. All experiments were implemented in Python 3.12.12 using PyTorch v2.7.1, MONAI v1.5.1, and ASTRA Toolbox v2.4.1. Training and inference were performed on an NVIDIA A4000 GPU with 24 GB of memory. Unless otherwise specified, all experiments used the random seed 42. The source code, configuration files, pretrained checkpoint, and patient-level data partitions
will be made available at  \url{code-released-upon-acceptance}. 

\subsection{Dataset and preprocessing}

The dataset used is the Mayo Clinic dataset~\cite{mccollough2016tu}, consisting of 3631 images with $512 \times 512$ pixels of human chests from 11 patients, which we have downscaled to $256 \times 256$ pixels and normalized in the range $[0, 1]$. In particular, we selected 3305 images from 10 patients for training and the remaining 326 from a single patient for testing. Figure \ref{fig:gtMayo}  shows representative  CT images from the test set, selected to illustrate different anatomical regions and structural characteristics. The corresponding enlarged regions displayed in the second row highlight fine anatomical details, including soft-tissue boundaries, pulmonary structures, and bone interfaces. Images belonging to the held-out test set are used as ground-truth references.

\begin{figure}[tbhp]
    \centering
    \includegraphics[width=0.18\linewidth]{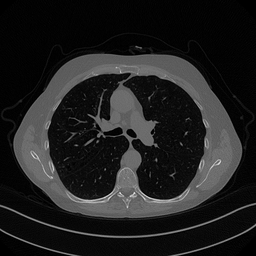}
    \includegraphics[width=0.18\linewidth]{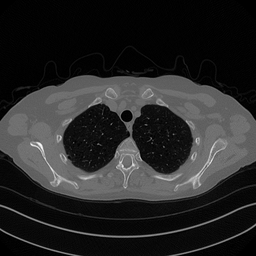}
    \includegraphics[width=0.18\linewidth]{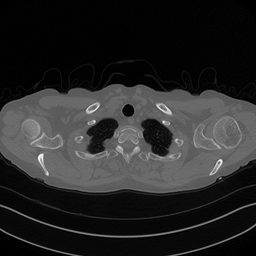}
    \includegraphics[width=0.18\linewidth]{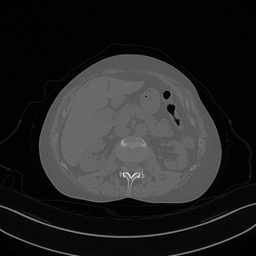}
    \includegraphics[width=0.18\linewidth]{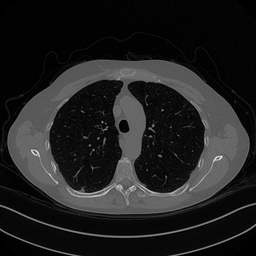} \\
    
    \includegraphics[trim= 28mm 38mm 32mm 37mm, clip, width=0.18\linewidth]{img/img00_gt.png}
    \includegraphics[trim= 35mm 45mm 25mm 30mm, clip, width=0.18\linewidth]{img/img01_gt.png}
    \includegraphics[trim= 20mm 25mm 40mm 50mm, clip,width=0.18\linewidth]{img/img07_gt.png}
    \includegraphics[trim= 40mm 45mm 20mm 30mm, clip,width=0.18\linewidth]{img/img03_gt.png}
    \includegraphics[trim= 30mm 15mm 30mm 60mm, clip,width=0.18\linewidth]{img/img06_gt.png}
    \caption{Representative axial CT images from the Mayo Clinic dataset used as
ground-truth references. The second row shows zoomed regions of interest.}
    \label{fig:gtMayo}
\end{figure}

\subsection{Diffusion model and training}
\label{subsec:diffusion_training}

As a diffusion prior for all the experiments, we consider the one released in \cite{evangelista2026physicsguidedflowmatchingct}, which was trained from scratch on the Mayo Clinic images. In particular, we employed a two-dimensional \texttt{DiffusionModelUNet} architecture provided by the MONAI Generative Models framework \cite{cardoso2022monai}. The network gets as input a batch of grey-scale $256 \times 256$ images and processes it via a four-level UNet architecture with $64$, $128$, $256$, and $512$ feature channels, respectively. Each resolution level contains two residual blocks. As is common practice, we employed self-attention only at the two deepest resolution levels, where the spatial dimensionality is smaller and long-range dependencies can be modeled at a reduced computational cost. The attention layers use $32$ channels per attention head. The network receives the noisy image $\x_t$ together with the corresponding diffusion timestep $t$, embedded through a sinusoidal embedding layer, and predicts a single-channel estimate of the noise $\eps_\Theta(\x_t,t)$ appearing in Eq.~\eqref{eq:loss_function}. 

Training was performed using a discrete DDPM forward process with $T=1000$ diffusion timesteps. We adopted a linear variance schedule, with $\beta_1=10^{-4}$ and $\beta_T=2\times10^{-2}$. At every training iteration, a clean image $\x_0$ and a timestep $t$ were sampled, with $t$ drawn uniformly from the $1000$ available diffusion steps. Before applying the forward diffusion process, the image intensities, originally normalized to $[0,1]$, were linearly mapped to $[-1,1]$. For any timestep $t$, a Gaussian noise realization $\eps_t \sim \mathcal{N}(\boldsymbol{0},\I)$ was then generated and added to the clean image according to the noise level associated with the sampled timestep. The resulting noisy image was provided to the U-Net, which was trained to recover the injected noise by minimizing the mean squared error
\begin{equation}
    \mathcal{L}_{DDPM}(\Theta) = \mathbb{E}_{\x_0,t,\eps_t} \left[ || \eps_\Theta(\x_t,t)-\eps_t ||_2^2 \right],
\end{equation}
consistently with the noise-prediction objective introduced in Eq.~\eqref{eq:loss_function}. The model parameters were optimized using Adam with a learning rate of $10^{-4}$ and a batch size of $16$, for a total of $100$ epochs. 


\subsection{Medical applications and imaging tasks}
\label{subsec:medical_tasks}

We evaluate CPF-DDNM on two medical imaging inverse problems: computed tomography reconstruction and single-image super-resolution. For CT, we consider both sparse-view and low-dose acquisition settings. Although these problems differ in the way the measurements are degraded, they all contain image components that are either unobservable or only weakly constrained by the acquired data, and therefore require prior information to obtain a plausible reconstruction.

\subsubsection{CT image reconstruction}
\label{subsec:ct_reconstruction}

Computed tomography reconstructs a cross-sectional image from projection measurements acquired at different viewing angles. In the two-dimensional parallel-beam geometry considered in this work, each projection angle defines a family of parallel X-ray paths through the object, while the detector records the corresponding line-integral measurements. The collection of all projections acquired over the angular range forms the sinogram.

Let $\x\in\mathbb{R}^{M^2}$ denote a ground-truth image and let $\K$ be the discretized parallel-beam projection operator. For each image, the simulated measurements are generated according to
\begin{equation}
\y=\K\x+\delta\left\|\K\x\right\|_{\infty}\e, \qquad \e\sim\mathcal{N}(\boldsymbol{0},\I),
\label{eq:ct_noise_model}
\end{equation}
where $\delta$ controls the relative noise level. Both $\K$ and $\K^\top$ were implemented using the \texttt{AstraToolbox}~\cite{van2015astra}. For an $M\times M$ image acquired using $N_{\theta}$ projection angles and a detector with $N_d$ elements, the corresponding discretized operator satisfies $\K\in\mathbb{R}^{N_{\theta} \cdot N_d\times M^2}$. In all our CT experiments, we used $M=256$ and $N_d=512$.

We considered two acquisition settings, corresponding to sparse-view and low-dose CT. 
In the sparse-view experiments, the number of acquired measurements is smaller than the number of unknown image pixels. In particular, for all the geometries considered in this work, $N_{\theta}N_d<M^2$, which implies that the system is underdetermined and that $\K$ necessarily has a non-trivial null-space. As the number of projection angles decreases, the amount of available information is further reduced, and the unresolved image subspace generally becomes larger. We test three acquisition geometries, denoted by $G_{60}$, $G_{45}$, and $G_{30}$, corresponding respectively to $N_{\theta}\in\{60,45,30\}$ projection angles uniformly distributed over $[0,180^\circ)$. A relative noise level of $\delta=0.01$ was used in all sparse-view experiments.

For the low-dose experiments, the parallel-beam geometry was instead fixed to $N_{\theta}=180$ projection angles uniformly distributed over $[0,180^\circ)$, so that the number and spatial arrangement of the measurements remained unchanged across the experiments. In actual low-dose CT acquisitions, the reduced photon count increases measurement uncertainty and decreases the signal-to-noise ratio of the projection data. We approximate this effect using the additive post-log Gaussian noise model in Eq.~\eqref{eq:ct_noise_model}, considering relative noise levels $\delta\in\{0.1,0.3,0.5\}$. This model should be interpreted as a simplified approximation of the quantum and electronic noise affecting CT measurements after logarithmic preprocessing, rather than as a complete physical simulation of the photon-counting acquisition process. Although the system matrix in this setting has more rows than columns, the nominal number of measurements $N_{\theta}N_d$ overestimates the amount of useful information available to the reconstruction, since detector bins corresponding to rays that do not intersect the image support produce zero measurements. For a square image with side length $M$, the projected width of the support at angle $\theta$ is approximately $M(|\cos\theta|+|\sin\theta|)$. Consequently, for uniformly distributed projection angles, the effective number of nonzero measurements can be approximated as
\begin{equation}
m_{\mathrm{eff}}\approx M\sum_{k=0}^{N_{\theta}-1}\left(|\cos\theta_k|+|\sin\theta_k|\right)\approx\frac{4}{\pi}MN_{\theta}.
\label{eq:effective_measurements}
\end{equation}
A complete proof of this result can be found in \cite{jorgensen2012quantifying}.
Since $\operatorname{rank}(\K)\leq m_{\mathrm{eff}}$, the condition $m_{\mathrm{eff}}<M^2$ is sufficient to imply $\operatorname{rank}(\K)<M^2$ and therefore $\operatorname{Ker}(\K)\neq\{\boldsymbol{0}\}$. Using the approximation in Eq.~\eqref{eq:effective_measurements}, this condition reduces to $N_{\theta}<\frac{\pi}{4}M$.

In our experiments, $M=256$ and $N_{\theta}=180$, while $\frac{\pi}{4}M\approx201.1$. According to this support-based dimensional estimate, the discretized projection operator considered in the low-dose experiments is therefore expected to have a non-trivial kernel. Independently of this exact null-space component, the added measurement noise also introduces substantial uncertainty along poorly conditioned and weakly observable directions of the CT operator.

\subsubsection{Medical image super-resolution}
\label{subsec:super_resolution}

Single-image super-resolution aims to reconstruct a high-resolution image from a spatially downsampled observation. In the setting considered in this work, each low-resolution pixel is obtained by averaging the intensities over a non-overlapping block of high-resolution pixels. This operation removes fine-scale anatomical information, including small structures and sharp tissue boundaries, which cannot be uniquely recovered from the observed image alone.

The acquisition model is written as $\y=\K\x+\e$, where $\x\in\mathbb{R}^{M^2}$ is the unknown high-resolution image, $\K$ denotes the block-averaging downsampling operator, and $\e$ represents additive measurement noise defined as in Eq.~\eqref{eq:ct_noise_model}. We considered both the noiseless case and a mildly noisy setting, with $\delta\in\{0,\ 0.01\}$.

For an integer scale factor $s>1$, the observed image contains $(M/s)^2$ pixels and the corresponding forward operator satisfies $\K\in\mathbb{R}^{(M/s)^2\times M^2}$. Since $(M/s)^2<M^2$, the resulting inverse problem is necessarily underdetermined and $\K$ has a non-trivial null-space. More precisely,
\begin{equation}
\dim\operatorname{Ker}(\K)=M^2-\operatorname{rank}(\K)\geq M^2-\left(\frac{M}{s}\right)^2.
\end{equation}
Multiple high-resolution images can therefore produce the same low-resolution observation, and the fine-scale image components removed by the downsampling operation must be inferred using prior information. This null-space ambiguity makes super-resolution a natural application for DDNM-based reconstruction and for the proposed CPF-DDNM strategy. In all super-resolution experiments, we used a scale factor $s=4$.
\subsection{Pseudo-inverse computation and algorithmic implementation}

The proposed method requires the application of the pseudo-inverse $\K^\dagger$ as in \eqref{eq:fusion_ddnm}. The way $\K^\dagger$ is evaluated depends on the structure of the forward operator, and this distinction is important for the  inverse problems considered in this work.\\
For super-resolution, the degradation operator is an average downsampling operator acting blockwise. In this case, the Singular Value Decomposition (SVD) of $\K$ is available in closed form, since each low-resolution pixel corresponds to the average of a non-overlapping high-resolution block. Therefore, the pseudo-inverse can be computed exactly by applying the corresponding blockwise SVD. This allows us to implement the DDNM correction with an exact $\K^\dagger$ in the super-resolution experiments.

For computed tomography, the forward operator $\K$ maps an image into its corresponding sinogram. At the image sizes considered in this work, the matrix representation of $\K$ is too large to be explicitly formed and factorized, and therefore the action of the Moore-Penrose pseudoinverse cannot be computed through an explicit SVD.
%
Therefore, rather than exactly constructing $\K^\dagger$, we approximate its action numerically using the conjugate-gradient least-squares (CGLS) algorithm initialized from the zero image. CGLS requires only applications of the forward operator $\K$ and its adjoint $\K^\top$, and is therefore well suited to large-scale tomographic reconstruction problems. For a fixed number $J$ of iterations, we denote the resulting approximate inverse action by
$$
\operatorname{CGLS}_J(\K,\boldsymbol{w}) \simeq \widetilde \K^\dagger_J \boldsymbol{w},
$$
where $\boldsymbol{w}$ is a generic measurement-space vector. 
%
Since to compute $\widetilde{\D}_t$ we only need $\K^\dagger \y$ and $\left(\I-\K^\dagger \K\right)\xot$, the DDNM correction can be implemented without explicitly constructing the projector $\K^\dagger$. In the CT experiments, we replace the exact pseudo-inverse action with its finite-iteration CGLS approximation and use
\begin{align}
    \D_t = \texttt{CGLS}_J(\K,\y) + \texttt{CGLS}_J(\K,\K\xot), \quad J=16.
\end{align}
%
Hence, the corrected estimates may be affected by two distinct sources of approximation: the error in the prior-driven prediction produced by the diffusion model and the numerical error associated with the finite-iteration approximation of $\K^\dagger$. Unlike the exact-pseudoinverse setting considered in the theoretical analysis, data consistency and the range/null-space decomposition are therefore satisfied only up to the accuracy of the CGLS approximation. We observe that finite-iteration CGLS introduces an implicit regularization effect, which can be relevant in the presence of measurement noise. 

The complete reconstruction procedure is summarized in Algorithm~\ref{alg:proposed_method}.

\begin{algorithm}[t]
\caption{CPF-DDNM posterior sampling}
\label{alg:proposed_method}
\begin{algorithmic}[1]
\REQUIRE Measurement $\by$, forward operator $\K$, pretrained denoiser $\bm{\epsilon}_\theta$, fusion schedule $\beta(t)$, number of CGLS iterations $J \in \mathbb{N}$
\STATE Initialize $\bx_T \sim \mathcal{N}(\boldsymbol{0},\I)$
\STATE Compute $\K^\dagger\by = \texttt{CGLS}_J(\K,\by)$
\STATE Set $\D_{\mathrm{prev}} = \texttt{None}$
\FOR{$t = T, T-\Delta t, \ldots, \Delta t$}
    \STATE Compute $\hat{\bx}_{0\vert t} = \left(\bx_t-\sigma_t\bm{\epsilon}_\theta(\bx_t,t)\right)/\alpha_t$
    \STATE Compute $\K^\dagger{\K}\hat{\bx}_{0\vert t} = \texttt{CGLS}_J\!\left(\K,\K\hat{\bx}_{0\vert t}\right)$
    \STATE Compute $\D_t = \K^\dagger\by + \hat{\bx}_{0\vert t} - \K^\dagger\K\hat{\bx}_{0\vert t}$ as in \eqref{eq:DDNM}
    \IF{$\D_{\mathrm{prev}} \neq \texttt{None}$}
        \STATE Set $\beta_t = \beta(t)$
        \STATE Compute $\widetilde{\D}_t = \D_t + \beta_t\left(\D_{\mathrm{prev}}-\D_t\right)$ as in \eqref{eq:fusion_intro}
    \ELSE
        \STATE Set $\widetilde{\D}_t = \D_t$
    \ENDIF
    \STATE Update $\bx_{t-\Delta t} = \alpha_{t-\Delta t}\widetilde{\D}_t + \sigma_{t-\Delta t}\bm{\epsilon}_\theta(\bx_t,t)$
    \STATE Set $\D_{\mathrm{prev}} = \D_t$
\ENDFOR
\RETURN $\bx_0$
\end{algorithmic}
\end{algorithm}

\subsection{Comparison methods, parameter settings, and evaluation metrics}

The proposed CPF-DDNM method was compared with DDNM \cite{wangzero}, DiffPIR \cite{zhu2023denoising}, Diffusion Posterior Sampling (DPS) \cite{chung2023diffusion}, and its Posterior Sampling Plus variant (PS+) \cite{chung2023diffusion}. DDNM represents the main baseline, since CPF-DDNM directly modifies its null-space completion strategy. For the sparse-view and low-dose CT experiments, we additionally considered the analytical filtered backprojection (FBP) method \cite{kak2001principles} and two model-based iterative reconstruction methods: the Simultaneous Iterative Reconstruction Technique (SIRT) \cite{gilbert1972iterative} and an $\ell_2$-TV regularized least-squares method solved using the Chambolle-Pock algorithm \cite{chambolle2011first,loli2021model}.

The parameters of the comparison methods were tuned on the training set and then kept fixed across all test images within the same task and acquisition setting. For DDNM and CPF-DDNM, we set $\Delta t=50$ in all experiments which, together with $T=1000$, results in $N=20$ reverse diffusion steps. In the CT experiments, the action of the pseudo-inverse $\K^\dagger$ was approximated using CGLS with $J = 16$ iterations, a setting that provided stable performance across the considered acquisition conditions. Note that limiting the number of CGLS iterations introduces an implicit regularization effect, since early termination reduces the amplification of measurement noise along poorly determined components of the inverse problem. This truncated CGLS approximation was therefore used for both DDNM and CPF-DDNM as an alternative to the explicit noise-handling strategy introduced by DDNM+.

The choice of the fusion coefficient $\beta_t$ is a crucial aspect of CPF-DDNM. Accordingly, we conducted an ablation study on sparse-view CT, reported in Section~\ref{ssec:ablation_beta}, considering both constant coefficients and time-dependent schedules. The experimental analysis indicates that the CT reconstruction tasks benefit from maintaining $\beta_t>1$ throughout the reverse trajectory, thereby preserving the extrapolative effect of the fusion. For all CT experiments, the fusion coefficient was defined according to the time-dependent schedule
\begin{equation}
\beta_t=1.8-0.6\sqrt{1-\frac{t}{T}},
\label{eq:ct_beta_schedule_2}
\end{equation}
which keeps $\beta_t$ in the extrapolative regime throughout the reverse diffusion process, with $\beta_t\in[1.2,1.8]$. For super-resolution, the most effective schedule instead gradually decreases toward zero during the reverse process and is defined as
\begin{equation}
\beta_t=1-\sqrt{1-\frac{t}{T}}.
\label{eq:ct_beta_schedule_1}
\end{equation}

Reconstruction quality was quantitatively evaluated using three complementary metrics: SSIM \cite{wang2003multiscale}, which measures structural similarity; PSNR, which quantifies pixel-wise reconstruction fidelity; and LPIPS \cite{zhang2018unreasonable}, which measures perceptual similarity in a learned feature space. Each reconstruction was repeated using five random initializations. For each test image, the metrics were first averaged across the five runs, after which the reported mean and standard deviation were computed across the test images.

\section{Results and discussion\label{sec:results}}

In this Section we present and discuss the results obtained on the experiments described in Section \ref{sec:expsetup}.

\subsection{Results on sparse-view CT reconstruction} 
The first numerical experiment addresses sparse-view CT reconstruction. We evaluate the considered methods under increasingly severe angular undersampling by using 60, 45, and 30 projection angles. The quantitative results are summarized in Table~\ref{tab:metrics_SpCT}.

\begin{table}[htbp]
\centering
\caption{Quantitative evaluation metrics computed over the entire test set across five runs with different initial noise realizations, for acquisition geometries with 30, 45, and 60 projection angles. In bold, the best results for each experiment.}
\footnotesize
\setlength{\tabcolsep}{3pt}
\renewcommand{\arraystretch}{1.05}
\begin{tabular}{llccc}
\toprule
 & \textbf{Method} & \textbf{SSIM ($\uparrow$)} & \textbf{LPIPS ($\downarrow$)} & \textbf{PSNR ($\uparrow$)} \\
\toprule
\multirow{8}{*}{$G_{30}$}
& FBP      & $0.407 \pm 0.014$ & $0.442 \pm 0.016$ & $22.83 \pm 0.46$ \\
& SIRT     & $0.636 \pm 0.021$ & $0.384 \pm 0.014$ & $26.52 \pm 0.55$ \\
& $\ell_2$-TV & $0.828 \pm 0.009$ & $0.132 \pm 0.015$ & $30.61 \pm 0.48$ \\
& DPS      & $0.756 \pm 0.018$ & $0.136 \pm 0.024$ & $26.62 \pm 0.72$ \\
& PS+      & $0.524 \pm 0.018$ & $0.231 \pm 0.013$ & $19.96 \pm 0.52$ \\
& DiffPIR  & $0.851 \pm 0.012$ & $0.085 \pm 0.011$ & $32.12 \pm 0.62$ \\
& DDNM     & $0.784 \pm 0.018$ & $0.132 \pm 0.011$ & $31.06 \pm 0.62$ \\
& CPF-DDNM & $\bm{0.851 \pm 0.009}$ & $\bm{0.083 \pm 0.015}$ & $\bm{32.40 \pm 0.52}$ \\
\midrule
\multirow{8}{*}{$G_{45}$}
& FBP      & $0.518 \pm 0.018$ & $0.364 \pm 0.012$ & $25.84 \pm 0.49$ \\
& SIRT     & $0.705 \pm 0.016$ & $0.322 \pm 0.011$ & $28.19 \pm 0.53$ \\
& $\ell_2$-TV & $0.860 \pm 0.006$ & $0.092 \pm 0.013$ & $32.62 \pm 0.45$ \\
& DPS      & $0.759 \pm 0.017$ & $0.135 \pm 0.022$ & $26.74 \pm 0.66$ \\
& PS+      & $0.525 \pm 0.018$ & $0.231 \pm 0.013$ & $19.98 \pm 0.53$ \\
& DiffPIR  & $0.888 \pm 0.011$ & $0.060 \pm 0.009$ & $34.10 \pm 0.63$ \\
& DDNM     & $0.824 \pm 0.016$ & $0.106 \pm 0.011$ & $32.67 \pm 0.61$ \\
& CPF-DDNM & $\bm{0.913 \pm 0.006}$ & $\bm{0.036 \pm 0.003}$ & $\bm{34.86 \pm 0.49}$ \\
\midrule
\multirow{8}{*}{$G_{60}$}
& FBP      & $0.602 \pm 0.019$ & $0.317 \pm 0.012$ & $27.63 \pm 0.50$ \\
& SIRT     & $0.758 \pm 0.015$ & $0.264 \pm 0.010$ & $29.55 \pm 0.54$ \\
& $\ell_2$-TV & $0.880 \pm 0.006$ & $0.076 \pm 0.010$ & $33.72 \pm 0.45$ \\
& DPS      & $0.759 \pm 0.016$ & $0.135 \pm 0.021$ & $26.74 \pm 0.61$ \\
& PS+      & $0.525 \pm 0.018$ & $0.230 \pm 0.013$ & $19.99 \pm 0.53$ \\
& DiffPIR  & $0.912 \pm 0.009$ & $0.043 \pm 0.007$ & $35.61 \pm 0.63$ \\
& DDNM     & $0.853 \pm 0.014$ & $0.084 \pm 0.010$ & $33.89 \pm 0.62$ \\
& CPF-DDNM & $\bm{0.938 \pm 0.006}$ & $\bm{0.021 \pm 0.003}$ & $\bm{37.27 \pm 0.66}$ \\
\bottomrule
\end{tabular}
\label{tab:metrics_SpCT}
\end{table}

As expected, the performance of most reconstruction methods improves as the number of acquired views increases. CPF-DDNM achieves the best mean values for all three evaluation metrics and for all the considered acquisition geometries, yielding the highest SSIM and PSNR values and the lowest LPIPS scores. For the most challenging 30-view geometry, the CPF-DDNM performance is very close to that of DiffPIR. Nevertheless, CPF-DDNM provides a clear improvement over the original DDNM baseline, increasing the SSIM from $0.7837$ to $0.8511$ and the PSNR from $31.06$ to $32.40\,\mathrm{dB}$, while reducing the LPIPS from $0.1322$ to $0.0829$.

For the 45- and 60-view geometries, the advantage of CPF-DDNM over the competing diffusion-based methods becomes more pronounced. With 45 views, CPF-DDNM improves the PSNR by approximately $0.77\,\mathrm{dB}$ over DiffPIR and reduces the LPIPS from 0.0601 to 0.0356, while also achieving a higher SSIM (0.9125 versus 0.8877). The improvement is even more evident with 60 views, where CPF-DDNM reaches a PSNR of $37.27\,\mathrm{dB}$, approximately $1.66\,\mathrm{dB}$ higher than DiffPIR, together with an LPIPS reduction from 0.0433 to 0.0208. Compared with the original DDNM baseline, the proposed fusion consistently improves all three metrics in both acquisition settings. Overall, the increasing performance gap observed from 45 to 60 views is consistent with the hypothesis that consecutive posterior fusion can more effectively exploit the additional measurement information available at denser angular samplings.

\begin{sidewaysfigure} 
\centering
\setlength{\tabcolsep}{0.5mm}

\begin{tabular}{@{}*{8}{c}@{}}

FBP &
SIRT &
$\ell_2$-TV &
DPS &
PS+ &
DiffPIR &
DDNM &
CPF-DDNM
\\[1mm]

\includegraphics[width=0.118\textwidth]{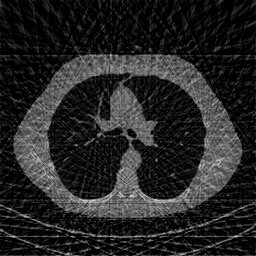} &
\includegraphics[width=0.118\textwidth]{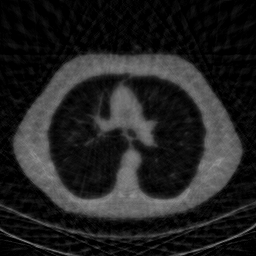} &
\includegraphics[width=0.118\textwidth]{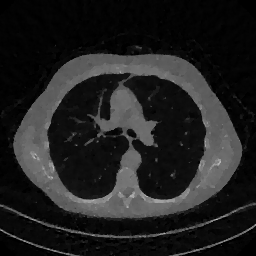} &
\includegraphics[width=0.118\textwidth]{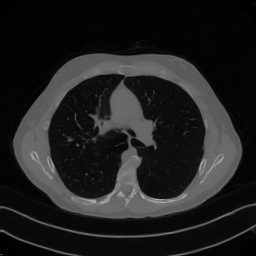} &
\includegraphics[width=0.118\textwidth]{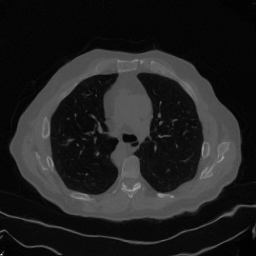} &
\includegraphics[width=0.118\textwidth]{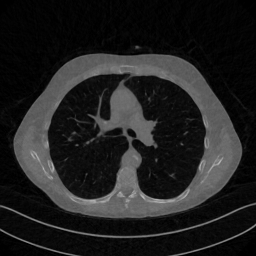} &
\includegraphics[width=0.118\textwidth]{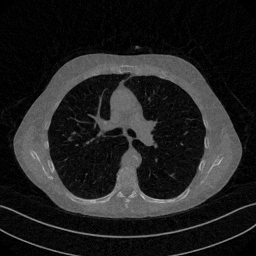} &
\includegraphics[width=0.118\textwidth]{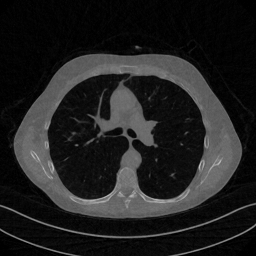}
\\ \noalign{\vskip -1mm}

\includegraphics[trim=28mm 38mm 32mm 37mm,clip,width=0.118\textwidth]{img/img00_30viste_fbp.png} &
\includegraphics[trim=28mm 38mm 32mm 37mm,clip,width=0.118\textwidth]{img/img00_30viste_sirt.png} &
\includegraphics[trim=28mm 38mm 32mm 37mm,clip,width=0.118\textwidth]{img/img00_30viste_tv.png} &
\includegraphics[trim=28mm 38mm 32mm 37mm,clip,width=0.118\textwidth]{img/img00_30viste_dps.png} &
\includegraphics[trim=28mm 38mm 32mm 37mm,clip,width=0.118\textwidth]{img/img00_30viste_psplus.png} &
\includegraphics[trim=28mm 38mm 32mm 37mm,clip,width=0.118\textwidth]{img/img00_30viste_diffpir.png} &
\includegraphics[trim=28mm 38mm 32mm 37mm,clip,width=0.118\textwidth]{img/img00_30viste_ddnm.png} &
\includegraphics[trim=28mm 38mm 32mm 37mm,clip,width=0.118\textwidth]{img/img00_30viste_cpf.png}
\\[1mm]

\includegraphics[width=0.118\textwidth]{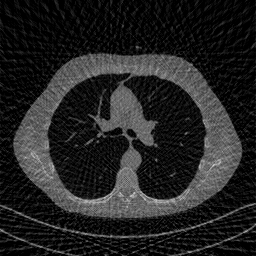} &
\includegraphics[width=0.118\textwidth]{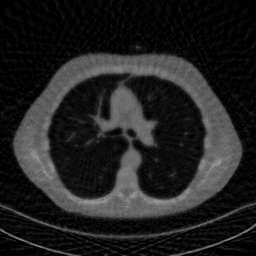} &
\includegraphics[width=0.118\textwidth]{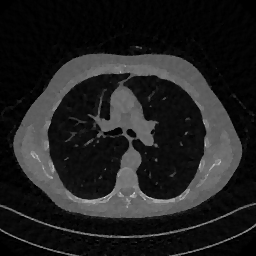} &
\includegraphics[width=0.118\textwidth]{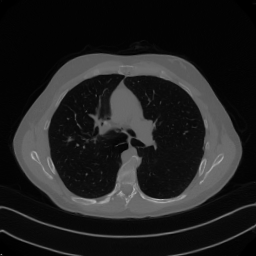} &
\includegraphics[width=0.118\textwidth]{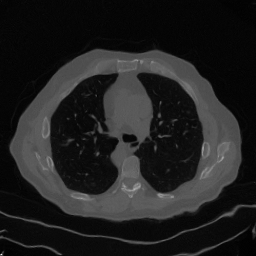} &
\includegraphics[width=0.118\textwidth]{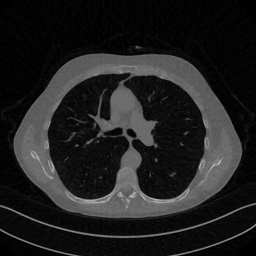} &
\includegraphics[width=0.118\textwidth]{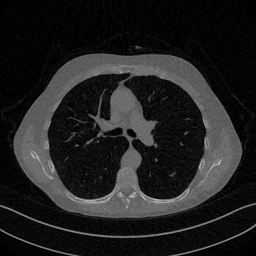} &
\includegraphics[width=0.118\textwidth]{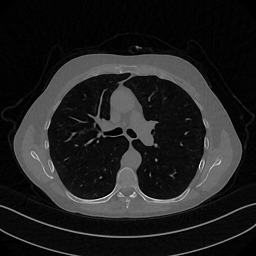}
\\ \noalign{\vskip -1mm}

\includegraphics[trim=28mm 38mm 32mm 37mm,clip,width=0.118\textwidth]{img/img00_45viste_fbp.png} &
\includegraphics[trim=28mm 38mm 32mm 37mm,clip,width=0.118\textwidth]{img/img00_45viste_sirt.png} &
\includegraphics[trim=28mm 38mm 32mm 37mm,clip,width=0.118\textwidth]{img/img00_45viste_tv.png} &
\includegraphics[trim=28mm 38mm 32mm 37mm,clip,width=0.118\textwidth]{img/img00_45viste_dps.png} &
\includegraphics[trim=28mm 38mm 32mm 37mm,clip,width=0.118\textwidth]{img/img00_45viste_psplus.png} &
\includegraphics[trim=28mm 38mm 32mm 37mm,clip,width=0.118\textwidth]{img/img00_45viste_diffpir.png} &
\includegraphics[trim=28mm 38mm 32mm 37mm,clip,width=0.118\textwidth]{img/img00_45viste_ddnm.png} &
\includegraphics[trim=28mm 38mm 32mm 37mm,clip,width=0.118\textwidth]{img/img00_45viste_cpf.png}
\\[1mm]

\includegraphics[width=0.118\textwidth]{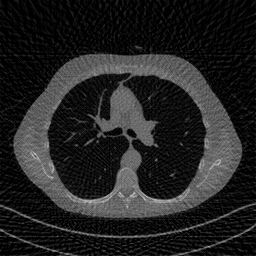} &
\includegraphics[width=0.118\textwidth]{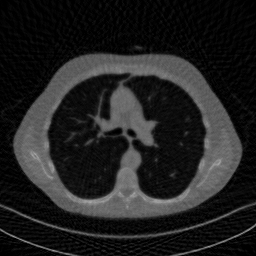} &
\includegraphics[width=0.118\textwidth]{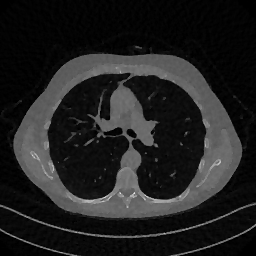} &
\includegraphics[width=0.118\textwidth]{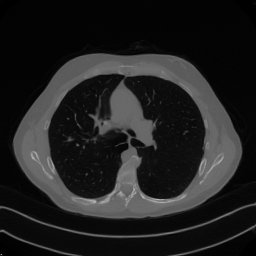} &
\includegraphics[width=0.118\textwidth]{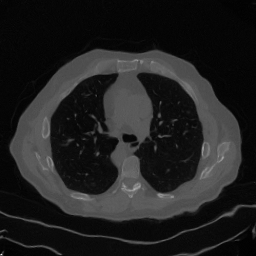} &
\includegraphics[width=0.118\textwidth]{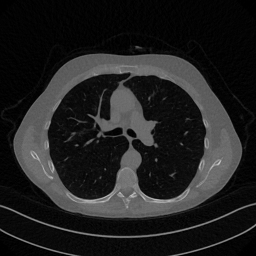} &
\includegraphics[width=0.118\textwidth]{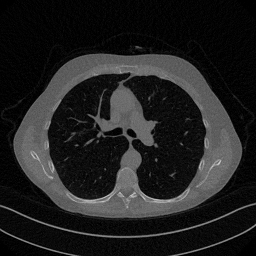} &
\includegraphics[width=0.118\textwidth]{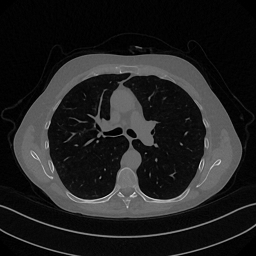}
\\ \noalign{\vskip -1mm}

\includegraphics[trim=28mm 38mm 32mm 37mm,clip,width=0.118\textwidth]{img/img00_60viste_fbp.png} &
\includegraphics[trim=28mm 38mm 32mm 37mm,clip,width=0.118\textwidth]{img/img00_60viste_sirt.png} &
\includegraphics[trim=28mm 38mm 32mm 37mm,clip,width=0.118\textwidth]{img/img00_60viste_tv.png} &
\includegraphics[trim=28mm 38mm 32mm 37mm,clip,width=0.118\textwidth]{img/img00_60viste_dps.png} &
\includegraphics[trim=28mm 38mm 32mm 37mm,clip,width=0.118\textwidth]{img/img00_60viste_psplus.png} &
\includegraphics[trim=28mm 38mm 32mm 37mm,clip,width=0.118\textwidth]{img/img00_60viste_diffpir.png} &
\includegraphics[trim=28mm 38mm 32mm 37mm,clip,width=0.118\textwidth]{img/img00_60viste_ddnm.png} &
\includegraphics[trim=28mm 38mm 32mm 37mm,clip,width=0.118\textwidth]{img/img00_60viste_cpf.png}
\end{tabular}
\caption{Qualitative comparison of sparse-view CT reconstructions obtained with the different methods. From top to bottom, the results correspond to 30, 45, and 60 projection angles. For each acquisition setting, the second row shows magnified details of the corresponding reconstructions.}
\label{fig:immagine_orizzontale}
\end{sidewaysfigure}

Figure~\ref{fig:immagine_orizzontale} provides a qualitative comparison of the reconstruction methods for the 30-, 45-, and 60-view geometries, shown from top to bottom, respectively. As expected, the reconstruction quality progressively improves as the number of projection angles increases, with better-defined anatomical structures and a reduction of streaking and blurring artifacts.
For the most challenging 30-view geometry, FBP and SIRT are strongly affected by reconstruction artifacts, whereas $\ell_2$-TV suppresses most of the streaks but produces an overly smoothed appearance. The diffusion-based methods provide more anatomically plausible reconstructions, although DPS and PS+ show limited structural fidelity. DiffPIR, DDNM, and CPF-DDNM recover the main anatomical structures more accurately, with CPF-DDNM providing slightly sharper boundaries and better preservation of fine details. 
The differences among the methods become clearer for the 45- and 60-view geometries. As more projection information becomes available, CPF-DDNM more effectively combines measurement consistency with the learned prior. 
These qualitative observations are consistent with the quantitative results reported in Table~\ref{tab:metrics_SpCT}. 
Overall, these results confirm the effectiveness of consecutive posterior fusion in improving the reconstruction quality of the original DDNM baseline.

\begin{figure}[ht]
\centering
\includegraphics[width=0.7\linewidth]{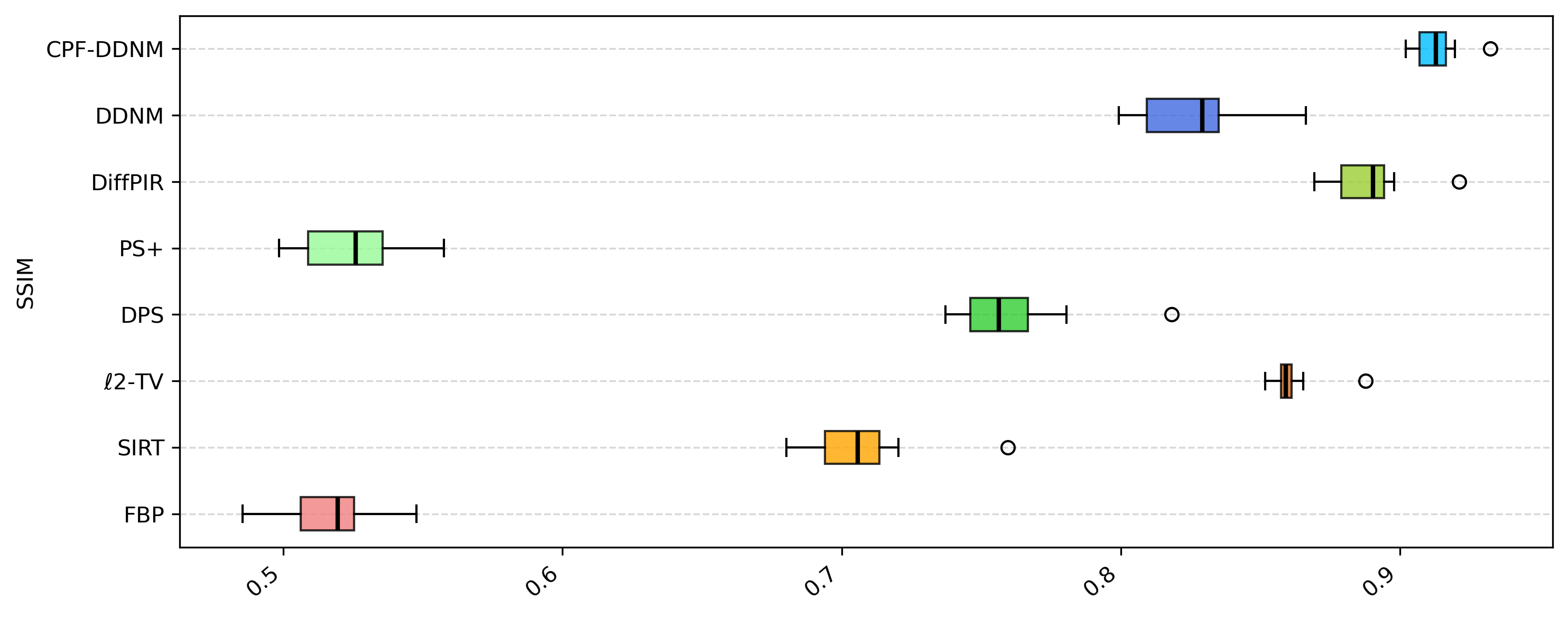}\\
\includegraphics[width=0.7\linewidth]{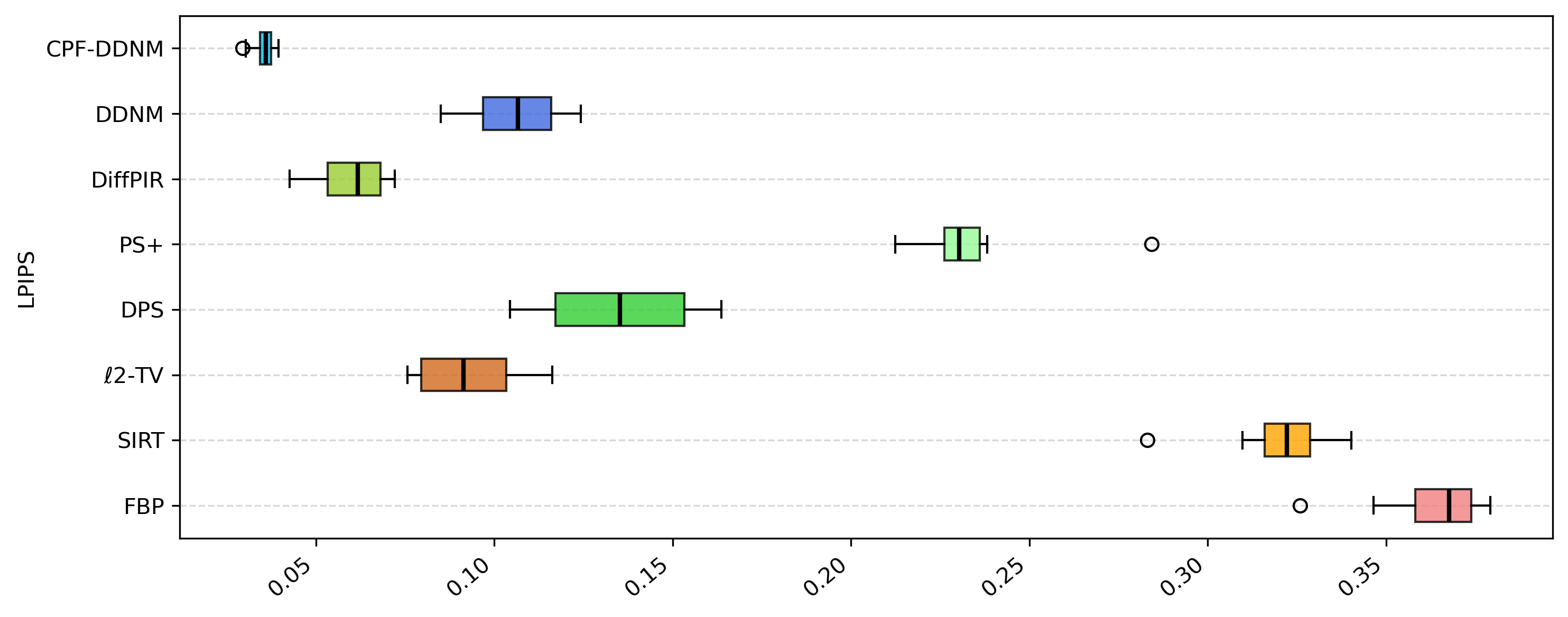}
\vspace{1mm}
\setlength{\tabcolsep}{0.5mm}
\begin{tabular}{@{}ccccc@{}}
\footnotesize FBP &
\footnotesize $\ell_2$-TV &
\footnotesize DiffPIR &
\footnotesize DDNM &
\footnotesize CPF-DDNM
\\
\includegraphics[width=0.185\linewidth]{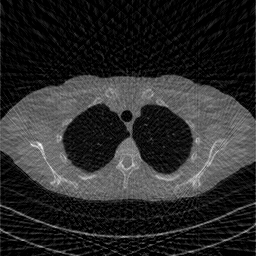} &
\includegraphics[width=0.185\linewidth]{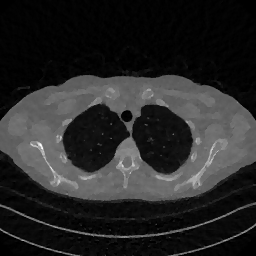} &
\includegraphics[width=0.185\linewidth]{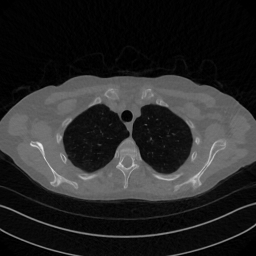} &
\includegraphics[width=0.185\linewidth]{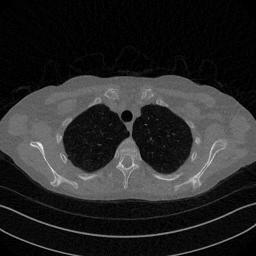} &
\includegraphics[width=0.185\linewidth]{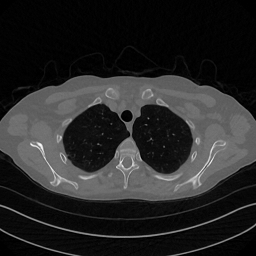}
\\ \noalign{\vskip -1mm}

\includegraphics[
    trim=35mm 45mm 25mm 30mm,
    clip,
    width=0.185\linewidth
]{img/img01_45viste_fbp.png} &
\includegraphics[
    trim=35mm 45mm 25mm 30mm,
    clip,
    width=0.185\linewidth
]{img/img01_45viste_tv.png} &
\includegraphics[
    trim=35mm 45mm 25mm 30mm,
    clip,
    width=0.185\linewidth
]{img/img01_45viste_diffpir.png} &
\includegraphics[
    trim=35mm 45mm 25mm 30mm,
    clip,
    width=0.185\linewidth
]{img/img01_45viste_ddnm.png} &
\includegraphics[
    trim=35mm 45mm 25mm 30mm,
    clip,
    width=0.185\linewidth
]{img/img01_45viste_cpf.png}
\end{tabular}
\caption{Quantitative and qualitative comparison for the 45-view sparse-CT geometry. The boxplots show the distributions of SSIM and LPIPS values computed over the entire test set across five runs with different initial noise realizations. Representative reconstructions and enlarged regions of interest are shown for visual comparison.}
\label{fig:SpCT_45}
\end{figure}
\clearpage

Figure~\ref{fig:SpCT_45} provides a more detailed analysis of the 45-view
geometry and is divided into two complementary parts. The upper part reports the boxplots of the SSIM and LPIPS values computed over the entire test set and across five runs with different initial noise realizations. CPF-DDNM achieves the best overall distributions, with higher SSIM values and lower LPIPS values than the competing methods. In particular, the improvement over DDNM is clearly visible for both metrics. 
The lower part presents a further qualitative comparison on a test image different from the one shown in Figure~\ref{fig:immagine_orizzontale}. The corresponding enlarged regions of interest highlight the differences in the recovery of the low-contrast area by the considered methods. 

\subsection{Ablation Study on the Fusion Coefficient}
\label{ssec:ablation_beta}

We investigated the influence of the fusion coefficient $\beta_t$ on the reconstruction dynamics by considering both constant coefficients and time-dependent schedules. The analysis was performed on sparse-view CT reconstruction with $45$ projection angles. For each choice of $\beta_t$, we monitored the evolution of SSIM, LPIPS, PSNR, and the squared data residual along the reverse diffusion trajectory averaged over the training set.

Following the geometric interpretation introduced in Section~\ref{ssec:cpf_ddnm}, we first considered the constant values
$$
\beta_t=\beta, \qquad \beta\in\{0,0.5,1,1.5\}.
$$
Here, $\beta=0$ corresponds to the non-fused DDNM approach, values in $(0, 1]$
reflects an interpolative regimes, whereas $\beta=1.5$ produces an extrapolative update. We additionally evaluated three time-dependent schedules. Let
\begin{align}
    \label{eq:normalized_progress}
    s_t = 1-\frac{t}{T}\in[0,1]
\end{align}
denote the normalized progress of the reverse diffusion process. Thus, $s_t=0$ at the beginning of sampling, when $t=T$, and $s_t=1$ at the end of the reverse trajectory, when $t=0$. The three schedules are defined as
\begin{align*}
    \text{Schedule A:}\qquad \beta_t^{A} &=1-\sqrt{s_t}, \\
    \text{Schedule B:}\qquad \beta_t^{B} &=1.8\left(1-\sqrt{s_t}\right), \\
    \text{Schedule C:}\qquad \beta_t^{C} &=1.8-0.6\sqrt{s_t}.
\end{align*}

All three schedules monotonically decrease the fusion coefficient as the reverse process approaches the clean image. However, they differ in the range of values attained by the fusion coefficient. Schedule~A satisfies
\begin{align*}
0\leq\beta_t^{A}\leq1,
\end{align*}
and therefore remains interpolative throughout the reverse process.
Schedule~B starts from the extrapolative value $\beta_T^{B}=1.8$, crosses the threshold $\beta_t=1$, and approaches zero toward the end of the reverse trajectory. Schedule~C instead satisfies
\begin{align*}
    1.2\leq\beta_t^{C}\leq1.8,
\end{align*}
and therefore preserves the extrapolative regime throughout the entire reverse trajectory while progressively reducing its strength.\\

\begin{figure}[ht]
    \centering
    \includegraphics[width=0.48\linewidth]{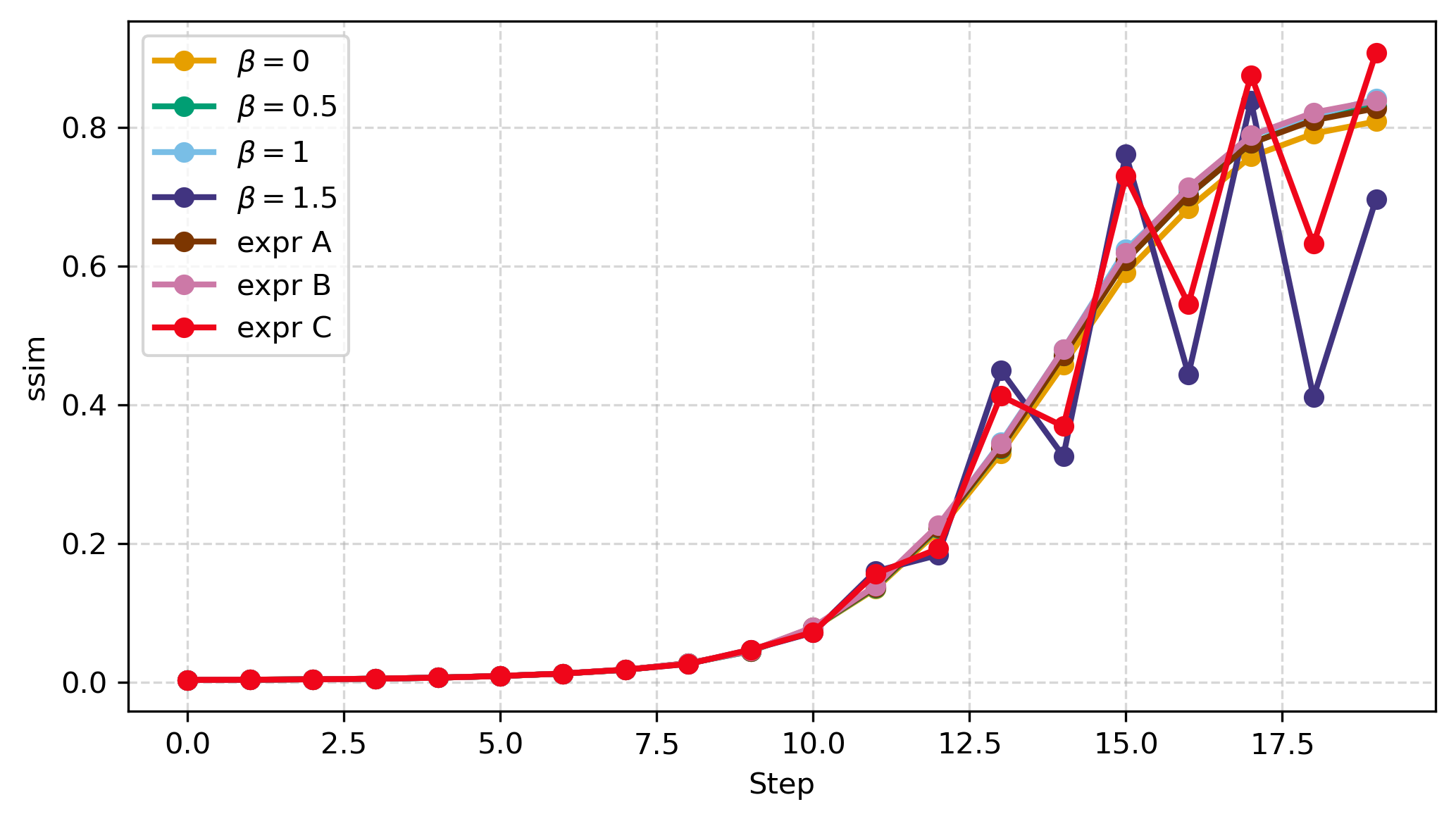}
    \includegraphics[width=0.48\linewidth]{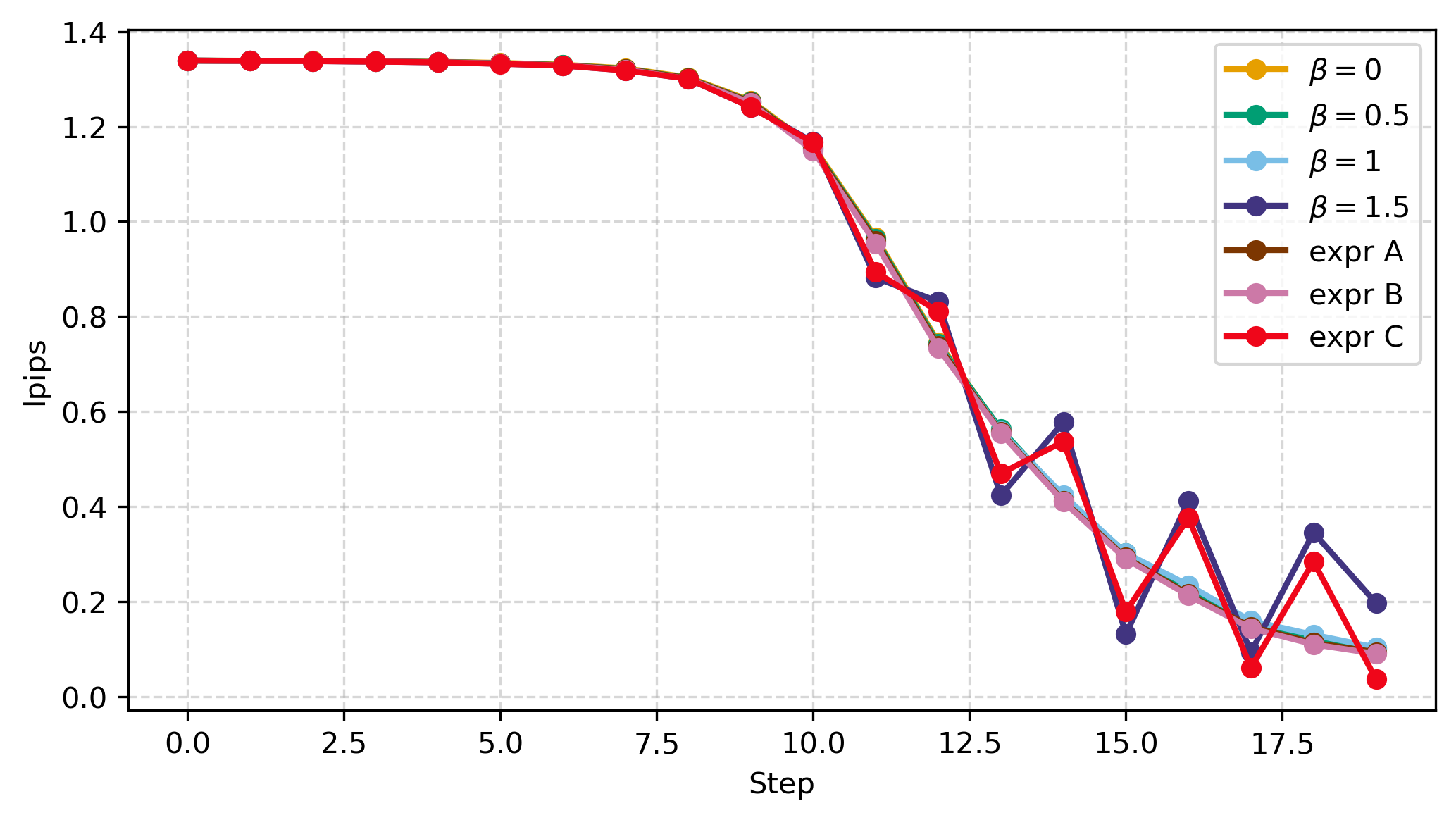}
    \includegraphics[width=0.48\linewidth]{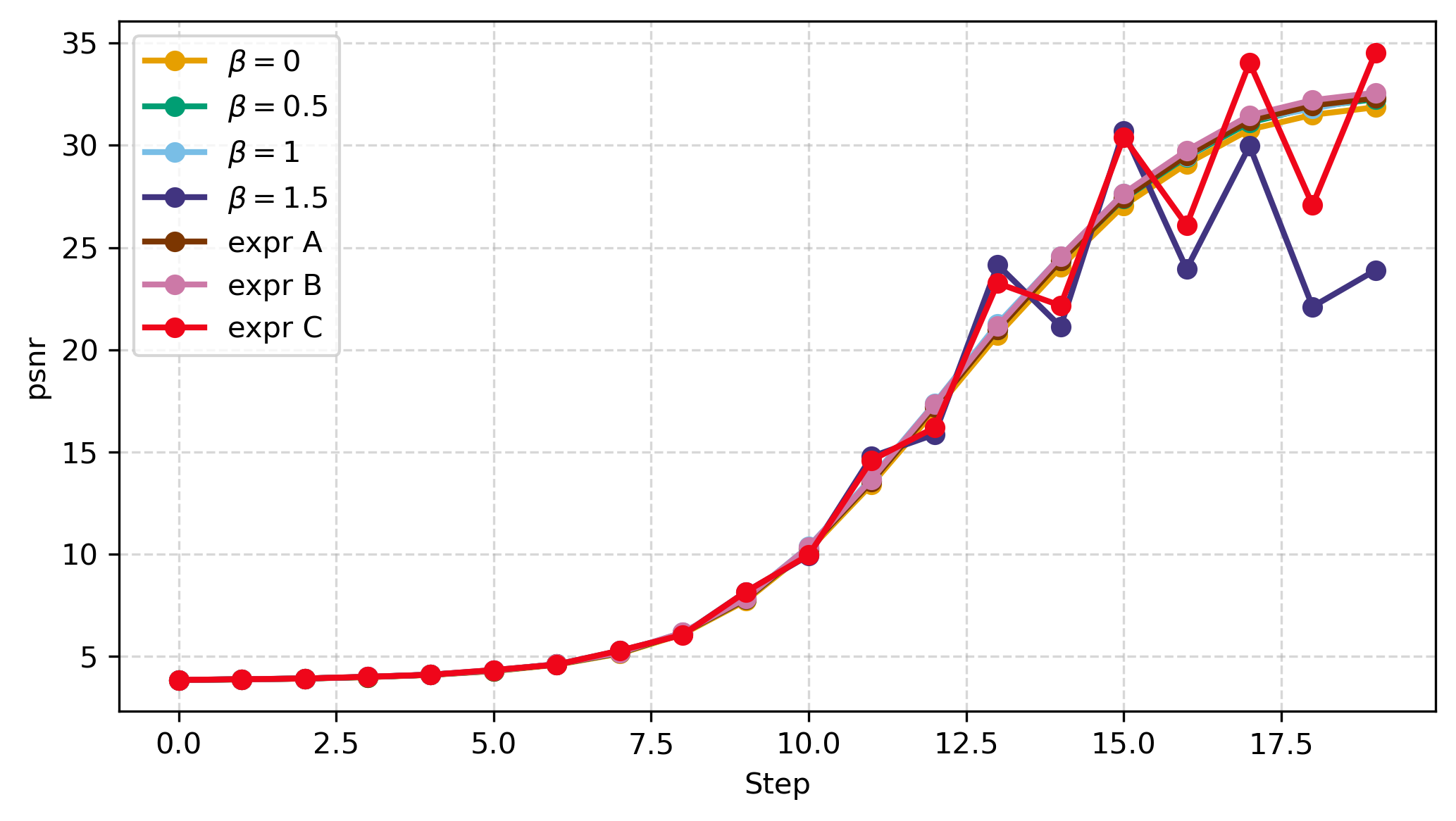}
    \includegraphics[width=0.48\linewidth]{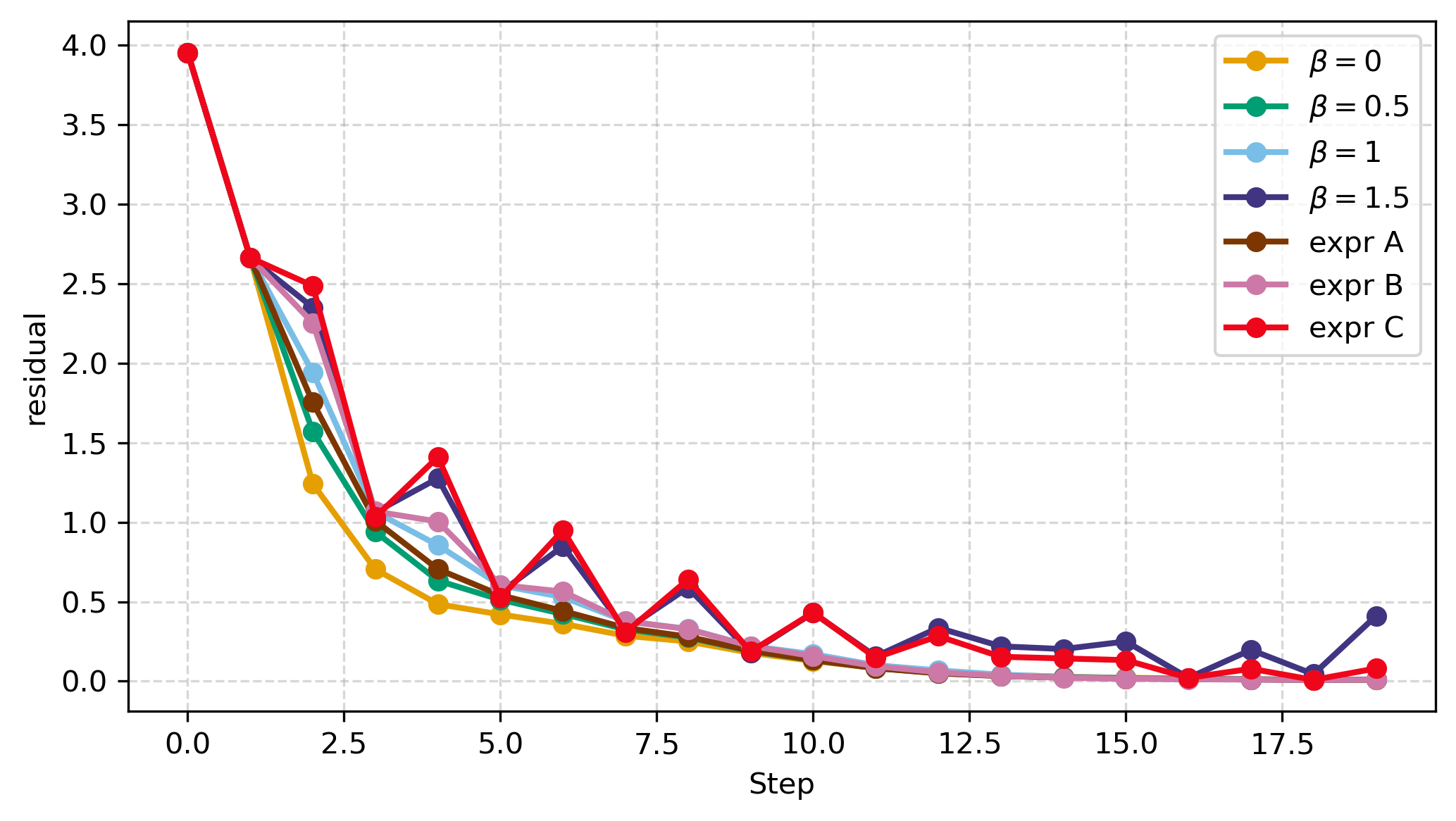}
    \caption{Evolution of SSIM, LPIPS, PSNR, and the squared data residual along the reverse diffusion trajectory for different constant coefficients and time-dependent schedules of $\beta_t$. Results refer to sparse-view CT reconstruction with $45$ projection angles.}
    \label{fig:ablation_beta}
\end{figure}

Figure~\ref{fig:ablation_beta} compares the reconstruction dynamics obtained with the different choices of $\beta_t$. During the initial part of the reverse trajectory, the metric curves are very similar. 
At these early stages of the reverse process, characterized by high diffusion noise, the consecutive denoised estimates still contain incomplete structural information, and their fusion has only a limited effect. The differences among the configurations become more evident during the final reverse steps, when the posterior estimates are more informative and their displacement defines a more meaningful reconstruction direction.

The constant choices $\beta_t \leq 1$ generally produce smoother trajectories than the extrapolative configurations. Since these choices do not introduce over-relaxation, their trajectories remain closer to those obtained with the standard DDNM update. 

A different behavior is observed for the constant coefficient $\beta_t=1.5$ and for Schedule~C, both of which remain extrapolative throughout the reverse process. In this regime, the additional term $\m_t = (\beta_t-1)\left(\D_{t+\Delta t}-\D_t\right)$ moves the reconstruction beyond the preceding posterior estimate along the direction identified by two consecutive reverse iterations. This over-relaxed update produces visible oscillations in SSIM, LPIPS, and PSNR during the final part of the trajectory. Nevertheless, the favorable phases of these oscillations reach higher SSIM and PSNR values and lower LPIPS values than the interpolative configurations.

The comparison between the constant choice $\beta_t=1.5$ and Schedule~C indicates that maintaining $\beta_t>1$ is not sufficient by itself.  Although both choices produce an oscillatory behavior, the constant coefficient may end the reverse process in an unfavorable phase,  resulting in a degradation of the final reconstruction. Schedule~C progressively reduces the extrapolation strength from $1.8$ to $1.2$ and attains the best final image-quality metrics among the considered configurations. 
This result motivates the use of Schedule~C in all the CT experiments reported in this work.

The squared data residual exhibits an overall decreasing trend for all configurations, although the extrapolative choices introduce noticeable oscillations during the final reverse steps. 
For some extrapolative configurations, however, the last reverse iteration may occur during an unfavorable phase of the oscillatory dynamics, leading to worse image-quality metrics than those obtained at the preceding iteration. Interestingly, the validation experiments showed that the iteration with the smallest data residual did not necessarily correspond to the best image-quality metrics. This observation suggests that the residual trajectory may contain useful information for identifying favorable terminal iterates, although the design of a systematic stopping criterion is left for future work. 

Overall, the ablation study suggests that a properly modulated extrapolative schedule can exploit the directional information contained in consecutive posterior estimates more effectively than either a purely interpolative fusion or a constant over-relaxation coefficient.

\subsection{Low-dose CT reconstruction}

In this section, we present the results obtained in the experiments simulating low-dose acquisitions as described in Section \ref{subsec:ct_reconstruction}. Figure \ref{fig:lowdose_barcharts} evaluates the robustness of the reconstruction methods to different noise levels. The bar plots report the mean SSIM and  LPIPS values obtained over the entire test set and across five runs with different initial noise realizations. The whiskers represent the corresponding standard deviations.\\
As expected, reconstruction quality generally deteriorates as the noise intensity increases, as indicated by decreasing SSIM values and increasing LPIPS values. CPF-DDNM achieves the highest SSIM values across all the considered noise levels. Regarding LPIPS, DiffPIR provides the best performance at the lowest noise level, $\delta=0.1$, whereas CPF-DDNM achieves the lowest values for $\delta=0.3$ and $\delta=0.5$. Interestingly, the relative advantage of CPF-DDNM becomes more pronounced as the measurement noise increases, particularly at the highest noise level. These results indicate that consecutive posterior fusion improves the robustness of the original DDNM formulation and remains effective under increasingly severe measurement noise. The qualitative results confirm that CPF-DDNM yields the best reconstruction, as further highlighted by the magnified view of a low-contrast region. In particular, noise is effectively suppressed while preserving sharp and well-defined object boundaries.

\begin{figure}[!ht]
\centering

\includegraphics[width=0.47\linewidth]{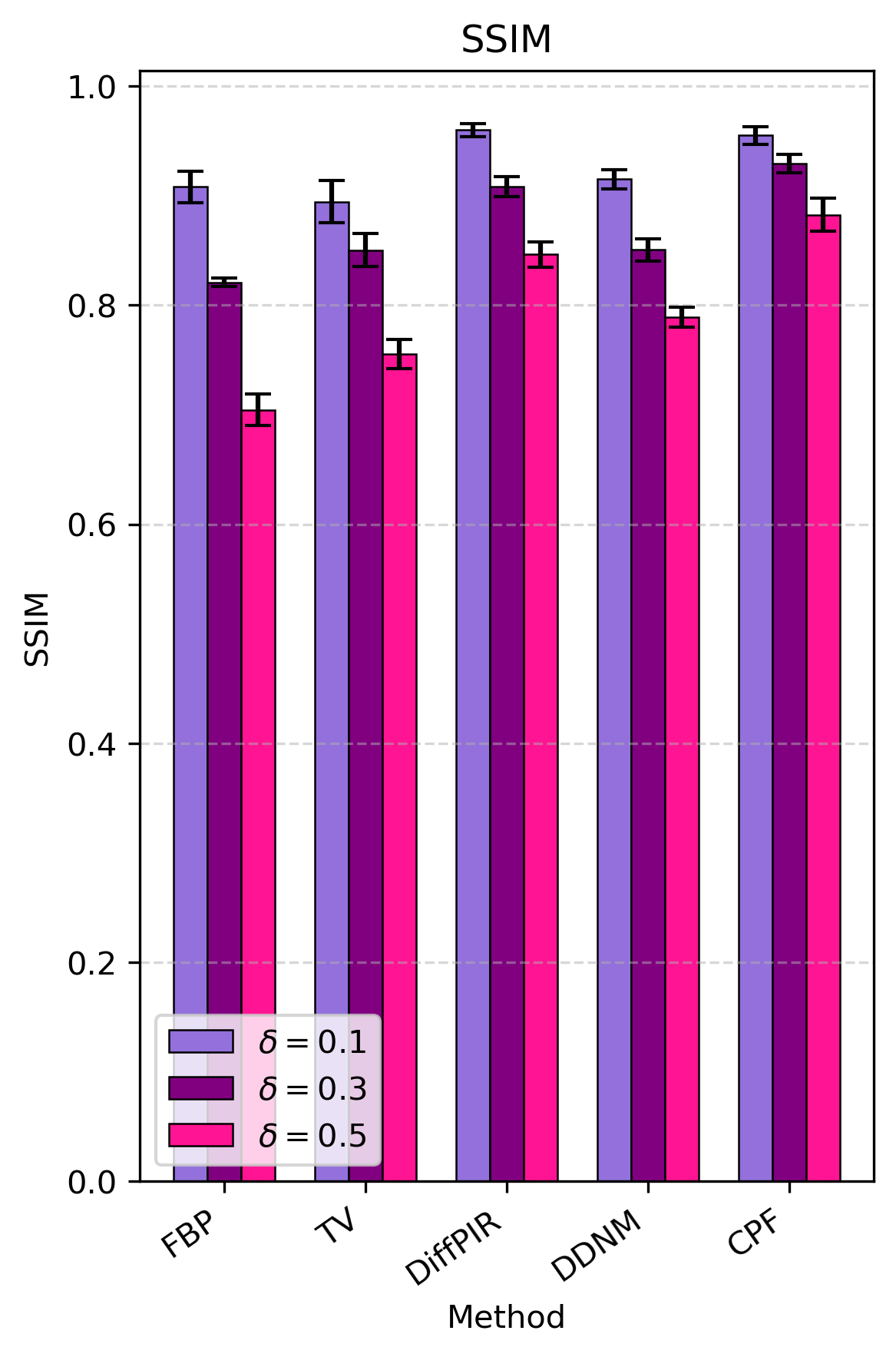}
\includegraphics[width=0.47\linewidth]{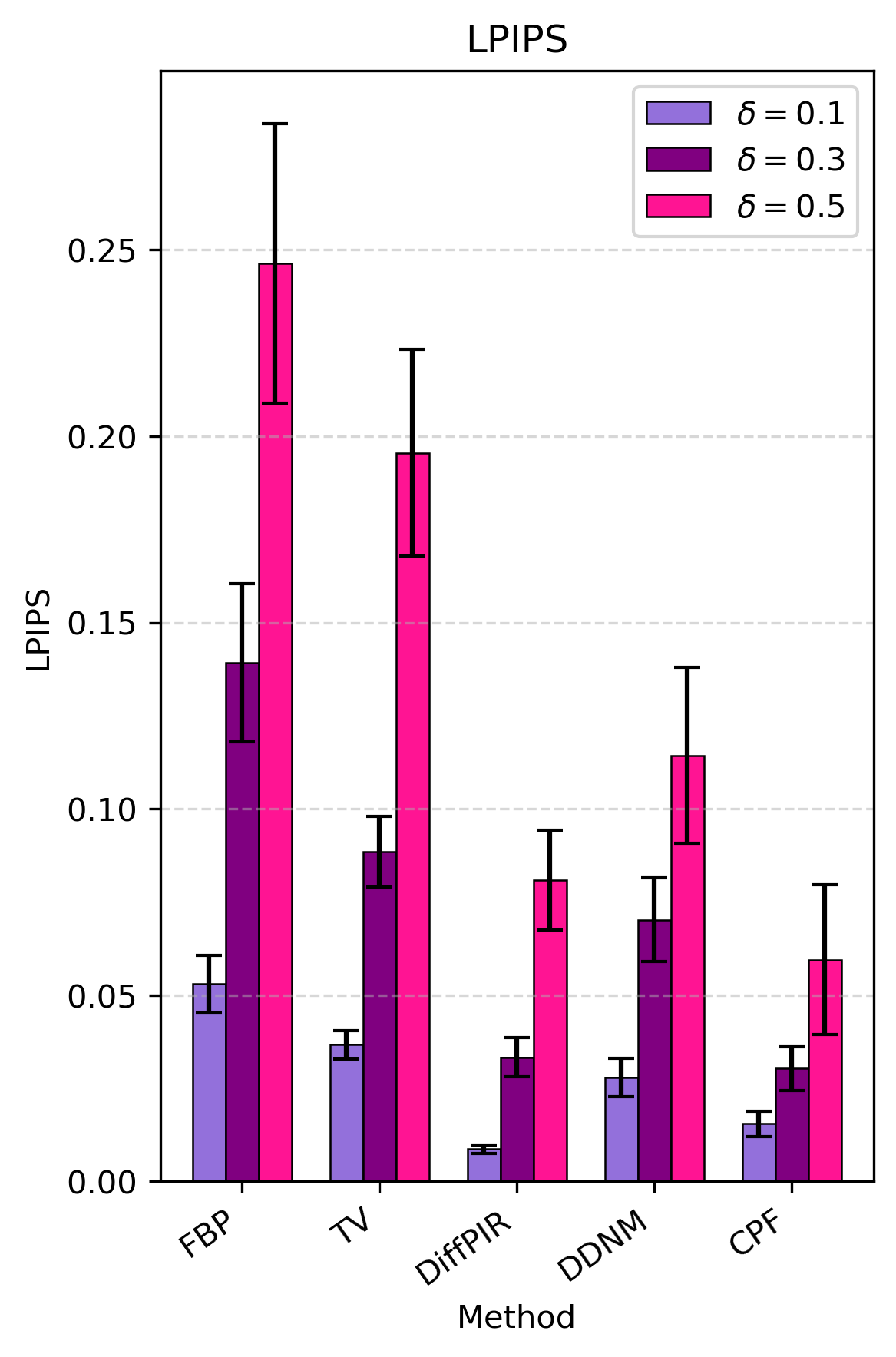}

\vspace{1mm}

\setlength{\tabcolsep}{0.5mm}
\begin{tabular}{@{}ccccc@{}}
\footnotesize FBP &
\footnotesize $\ell_2$-TV &
\footnotesize DiffPIR &
\footnotesize DDNM &
\footnotesize CPF-DDNM
\\

\includegraphics[width=0.185\linewidth]{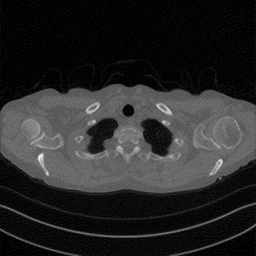} &
\includegraphics[width=0.185\linewidth]{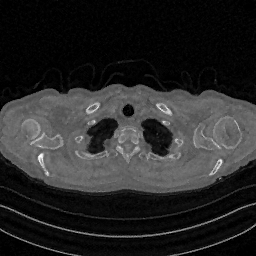} &
\includegraphics[width=0.185\linewidth]{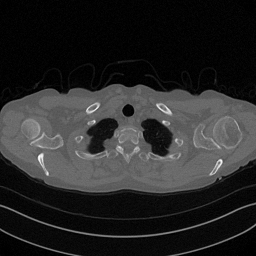} &
\includegraphics[width=0.185\linewidth]{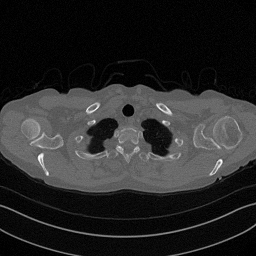} &
\includegraphics[width=0.185\linewidth]{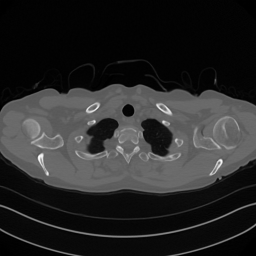}
\\ \noalign{\vskip -1mm}

\includegraphics[
    trim=20mm 25mm 40mm 50mm,
    clip,
    width=0.185\linewidth
]{img/img07_lowdose_fbp.png} &
\includegraphics[
    trim=20mm 25mm 40mm 50mm,
    clip,
    width=0.185\linewidth
]{img/img07_lowdose_tv.png} &
\includegraphics[
    trim=20mm 25mm 40mm 50mm,
    clip,
    width=0.185\linewidth
]{img/img07_lowdose_diffpir.png} &
\includegraphics[
    trim=20mm 25mm 40mm 50mm,
    clip,
    width=0.185\linewidth
]{img/img07_lowdose_ddnm.png} &
\includegraphics[
    trim=20mm 25mm 40mm 50mm,
    clip,
    width=0.185\linewidth
]{img/img07_lowdose_cpf.png}
\end{tabular}

\caption{Robustness to measurement noise in low-dose CT reconstruction. The upper part reports the mean SSIM and LPIPS values over the complete test set for noise levels $\delta=0.1$, $0.3$, and $0.5$, with whiskers indicating the corresponding standard deviations. The lower part presents representative reconstructions and enlarged regions of interest for $\delta=0.3$, comparing the preservation of anatomical structures and fine details across the considered methods.}
\label{fig:lowdose_barcharts}
\end{figure}

\subsection{ Super-resolution of CT images}

\begin{table}[ht]
\centering
\caption{Quantitative comparison for $\times 4$ super-resolution in the noiseless setting and with noise level $\delta=0.01$. Results are reported as mean $\pm$ standard deviation over the test images, after averaging each metric over five random initializations per image. Bold values indicate the best result for each setting.}
\label{tab:metrics_SuperRes}

\footnotesize
\setlength{\tabcolsep}{3pt}
\renewcommand{\arraystretch}{1.05}

\begin{tabular}{@{}llccc@{}}
\toprule
\textbf{Noise} & \textbf{Method}
& \textbf{SSIM $\uparrow$}
& \textbf{LPIPS $\downarrow$}
& \textbf{PSNR $\uparrow$} \\
\midrule

\multirow{5}{*}{$\delta=0$}
& DPS
& $0.497 \pm 0.055$
& $0.288 \pm 0.017$
& $19.60 \pm 1.28$ \\

& PS+
& $0.774 \pm 0.049$
& $0.117 \pm 0.015$
& $28.40 \pm 1.66$ \\

& DiffPIR
& $0.863 \pm 0.028$
& $0.093 \pm 0.014$
& $30.18 \pm 1.32$ \\

& DDNM
& $0.856 \pm 0.031$
& $\bm{0.076 \pm 0.014}$
& $\bm{31.40 \pm 1.60}$ \\

& CPF-DDNM
& $\bm{0.872 \pm 0.028}$
& $0.101 \pm 0.014$
& $31.35 \pm 1.60$ \\

\midrule

\multirow{5}{*}{$\delta=0.01$}
& DPS
& $0.495 \pm 0.055$
& $0.288 \pm 0.017$
& $19.60 \pm 1.28$ \\

& PS+
& $0.771 \pm 0.055$
& $0.118 \pm 0.014$
& $28.23 \pm 1.60$ \\

& DiffPIR
& $\bm{0.839 \pm 0.028}$
& $\bm{0.111 \pm 0.019}$
& $\bm{29.59 \pm 1.27}$ \\

& DDNM
& $0.808 \pm 0.028$
& $0.131 \pm 0.017$
& $29.03 \pm 1.30$ \\

& CPF-DDNM
& $0.828 \pm 0.027$
& $0.142 \pm 0.014$
& $29.19 \pm 1.39$ \\

\bottomrule
\end{tabular}%

\end{table}
The final numerical experiment evaluates the considered methods for the $\times 4$ image super-resolution setting described in Section~\ref{sec:expsetup}. Table~\ref{tab:metrics_SuperRes} reports the mean SSIM, LPIPS, and PSNR values computed over the complete test set, while Figure~\ref{fig:superRes} shows two representative reconstructions for the noisy case.

In the noiseless setting, CPF-DDNM achieves the highest SSIM value and a PSNR essentially comparable to that of DDNM, indicating an effective preservation of structural information. \\
A different behavior is observed when measurement noise is introduced. According to the test-set averages reported in  Table~\ref{tab:metrics_SuperRes}, DiffPIR achieves the best values for all three metrics. CPF-DDNM nevertheless improves upon the original DDNM baseline in terms of SSIM and PSNR, although it obtains a higher LPIPS value. Thus, from a purely quantitative perspective, DiffPIR provides the best overall performance in the noisy setting.

The qualitative comparison in Figure~\ref{fig:superRes}, however, reveals a marked discrepancy with this quantitative ranking. In both displayed examples, the DiffPIR reconstructions appear noticeably smoother, with some low-contrast tissues being less clearly resolved in the enlarged regions of interest. CPF-DDNM instead produces sharper tissue boundaries and preserves several localized details that are less visible in the DiffPIR reconstructions.
Among the methods shown, CPF-DDNM therefore provides the most convincing visual results for these two noisy examples, despite the lower average metric scores reported in Table~\ref{tab:metrics_SuperRes}.
This discrepancy illustrates that global image-quality metrics such as LPIPS may also be insensitive to localized differences that are clinically or anatomically relevant.

\begin{figure}[!ht]
\centering

\setlength{\tabcolsep}{0.5mm}
\begin{tabular}{@{}ccccc@{}}
\footnotesize DPS &
\footnotesize PS+ &
\footnotesize DiffPIR &
\footnotesize DDNM &
\footnotesize CPF-DDNM
\\

\includegraphics[width=0.185\linewidth]{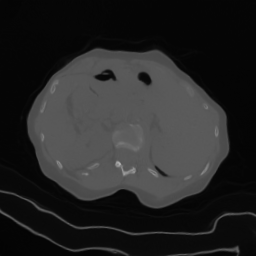} &
\includegraphics[width=0.185\linewidth]{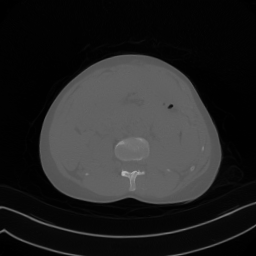} &
\includegraphics[width=0.185\linewidth]{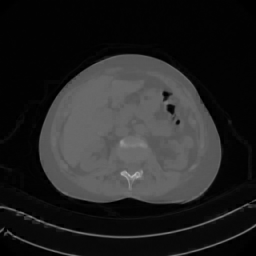} &
\includegraphics[width=0.185\linewidth]{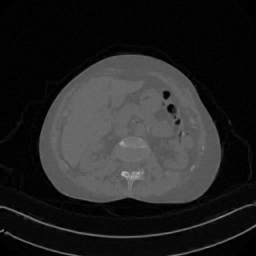} &
\includegraphics[width=0.185\linewidth]{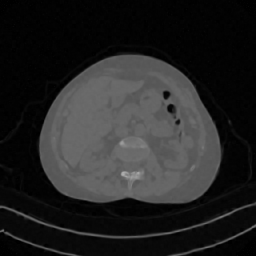}
\\ \noalign{\vskip -1mm}

\includegraphics[
    trim=40mm 45mm 20mm 30mm,
    clip,
    width=0.185\linewidth
]{img/img03_superres_dps.png} &
\includegraphics[
    trim=40mm 45mm 20mm 30mm,
    clip,
    width=0.185\linewidth
]{img/img03_superres_psplus.png} &
\includegraphics[
    trim=40mm 45mm 20mm 30mm,
    clip,
    width=0.185\linewidth
]{img/img03_superres_diffpir.png} &
\includegraphics[
    trim=40mm 45mm 20mm 30mm,
    clip,
    width=0.185\linewidth
]{img/img03_superres_ddnm.png} &
\includegraphics[
    trim=40mm 45mm 20mm 30mm,
    clip,
    width=0.185\linewidth
]{img/img03_superres_cpf.png}
\\[1mm]

\includegraphics[width=0.185\linewidth]{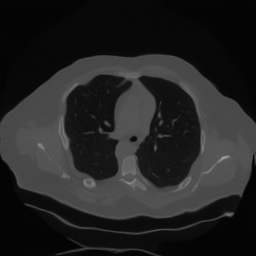} &
\includegraphics[width=0.185\linewidth]{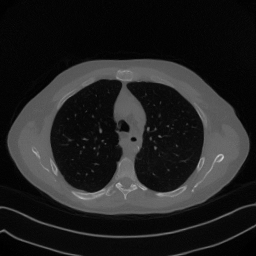} &
\includegraphics[width=0.185\linewidth]{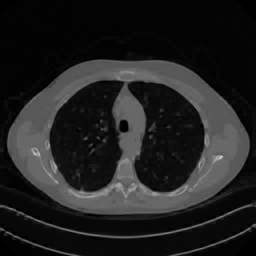} &
\includegraphics[width=0.185\linewidth]{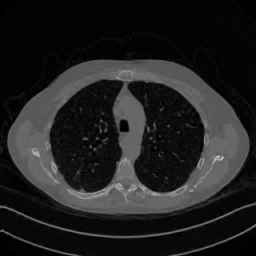} &
\includegraphics[width=0.185\linewidth]{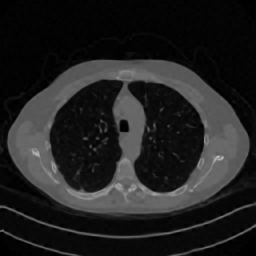}
\\ \noalign{\vskip -1mm}

\includegraphics[
    trim=30mm 15mm 30mm 60mm,
    clip,
    width=0.185\linewidth
]{img/img06_superres_dps.png} &
\includegraphics[
    trim=30mm 15mm 30mm 60mm,
    clip,
    width=0.185\linewidth
]{img/img06_superres_psplus.png} &
\includegraphics[
    trim=30mm 15mm 30mm 60mm,
    clip,
    width=0.185\linewidth
]{img/img06_superres_diffpir.png} &
\includegraphics[
    trim=30mm 15mm 30mm 60mm,
    clip,
    width=0.185\linewidth
]{img/img06_superres_ddnm.png} &
\includegraphics[
    trim=30mm 15mm 30mm 60mm,
    clip,
    width=0.185\linewidth
]{img/img06_superres_cpf.png}
\end{tabular}

\caption{Qualitative comparison of the different methods for $\times 4$ super-resolution in the presence of measurement noise with $\delta=0.01$. Two representative images from the test set are shown together with the corresponding reconstructions and enlarged regions of interest.}
\label{fig:superRes}
\end{figure}

\section{Conclusions}
\label{sec:conclusions}
In this work, we introduced CPF-DDNM, a lightweight inference-time strategy for improving the prior-driven recovery of unobservable image components in DDNM-based inverse solvers. The method exploits information already available along the reverse diffusion trajectory by combining consecutive measurement-aware estimates, without modifying or retraining the diffusion prior.

The range/null-space decomposition of DDNM provides a direct geometric interpretation of the proposed fusion. With an exact Moore--Penrose pseudo-inverse, consecutive DDNM corrections share the same measurement-determined component and differ only in the null-space of the forward operator. CPF-DDNM therefore modifies exclusively the prior-driven component while preserving the information determined by the measurements. Thus, for noiseless data, CPF-DDNM preserves exact data consistency. 
Our local error analysis further provides conditions under which the extrapolative optimal fusion coefficient can yield a small null-space recovery error.

Experiments demonstrate that the benefits of consecutive posterior fusion depend on both the forward operator and the fusion schedule. The largest gains were observed in sparse-view CT, where CPF-DDNM consistently improved upon DDNM across acquisition geometries and noise levels, and outperformed the considered diffusion-based inverse solvers. In super-resolution, CPF-DDNM improved structural similarity over DDNM, although DiffPIR achieved better average performance for some metrics in the noisy setting. The method was effective with both numerically approximated pseudo-inverses in CT and an analytically available pseudo-inverse in super-resolution.

CPF-DDNM requires no additional denoiser evaluations or pseudo-inverse operations compared with DDNM, adding only the storage of the preceding correction and a linear combination. The main remaining limitation is the selection of the fusion schedule, which is currently determined on validation data and depends on the forward operator. Future work will investigate adaptive inference-time strategies based on quantities such as consecutive correction variations and data residuals, as well as plugging the CPF strategy to different posterior sampling methods.



\bibliographystyle{plain}
\bibliography{biblio}
\end{document}